\documentclass[final]{elsarticle}

\journal{ArXiv}

\usepackage{longtable}
\usepackage{booktabs}
 
\usepackage{siunitx}

\makeatletter
\def\set@curr@file#1{\def\@curr@file{#1}}
\makeatother

\usepackage[utf8]{inputenc}
\usepackage[ruled,vlined,linesnumbered]{algorithm2e}
\usepackage{array}
\usepackage{amsthm}
\usepackage{amsmath}
\usepackage{amssymb}
\usepackage{appendix}
\usepackage{csquotes}
\usepackage{subfigure}
\usepackage{url}

\MakeOuterQuote{"}
\usepackage{natbib}
\usepackage{mathtools}
\usepackage{multicol}
\usepackage{tikz}
\usetikzlibrary{positioning, fit, calc, shapes}

\newcommand{\R}{\mathbb{R}}
\newcommand{\K}{\mathcal{K}}
\renewcommand{\S}{\mathcal{S}}
\newcommand{\refp}[1]{(\ref{#1})}

\newtheorem{theorem}{Theorem}[section]
\newtheorem{definition}{Definition}[section]
\newtheorem{lemma}{Lemma}[section]

\makeatletter

\patchcmd{\NAT@test}{\else \NAT@nm}{\else \NAT@nmfmt{\NAT@nm}}{}{}

\DeclareRobustCommand\citepos
  {\begingroup
   \let\NAT@nmfmt\NAT@posfmt
   \NAT@swafalse\let\NAT@ctype\z@\NAT@partrue
   \@ifstar{\NAT@fulltrue\NAT@citetp}{\NAT@fullfalse\NAT@citetp}}

\let\NAT@orig@nmfmt\NAT@nmfmt
\def\NAT@posfmt#1{\NAT@orig@nmfmt{#1's}}

\newcolumntype{Y}[1]{>{\hangindent=1em \raggedright \let\newline\\\arraybackslash}p{#1}}

\begin{document}

\begin{frontmatter}

\title{Conditional Hierarchical Bayesian Tucker Decomposition for Genetic Data Analysis}

\author[inst1]{Adam Sandler}
\ead{adamsandler2021@u.northwestern.edu}

\affiliation[inst1]{organization={Engineering Sciences and Applied Mathematics},
            addressline={Northwestern University}, 
            city={Evanston},
            state={IL},
            country={US}}

\author[inst2]{Diego Klabjan}
\ead{d-klabjan@northwestern.edu}

\affiliation[inst2]{organization={Industrial Engineering and Management Sciences},
            addressline={Northwestern University}, 
            city={Evanston},
            state={IL},
            country={US}}
            
\author[inst3]{Yuan Luo}
\ead{yuan.luo@northwestern.edu}

\affiliation[inst3]{organization={Preventive Medicine (Health and Biomedical Informatics)},
            addressline={Northwestern University}, 
            city={Chicago},
            state={IL},
            country={US}}


\begin{abstract}
We analyze large, multi-dimensional, sparse counting data sets, finding unsupervised groups to provide unique insights into genetic data. We create gene and biological pathway groups based on patients' variants to find common risk factors for four common types of cancer (breast, lung, prostate, and colorectal) and autism spectrum disorder. To accomplish this, we extend latent Dirichlet allocation to multiple dimensions and design distinct methods for hierarchical topic modeling. We find that our conditional hierarchical Bayesian Tucker decomposition models are more coherent than baseline models.
\end{abstract}


\end{frontmatter}

\section{Introduction}
\label{Introduction}

Genetic and environmental factors cause autism spectrum disorder (ASD) and cancer. We create unsupervised groupings of relevant genes and their biological pathways to examine patients' genetic variants to find common risk factors. Learning about patients using genetic data often includes more features than observations, which makes direct supervised learning difficult. Existing methods focus on individual gene markers or curated lists, while groups of genes may provide more insights (see Section \ref{Literature Review}).

Latent Dirichlet allocation (LDA) is a Bayesian model commonly used in natural language processing. In this context, it groups documents by their common words. Words frequently used in similar documents are placed in the same group. In the biomedical context, we find genetic variants in similar patients. We extend this model to incorporate multiple dimensions (conditional Bayesian Tucker decomposition) and use gene and pathway information. Pathways provide a priori knowledge of which genes work together to perform specific functions, while the unsupervised gene groups provide distinct insights about the given patients. These models are quicker to compute than exact methods and better represent rare words or genes than non-Bayesian models (a uniform prior is typically used to account for rare instances).

We design distinct methods for incorporating hierarchical topic modeling based on nested Chinese restaurant processes (CRP) and Pachinko Allocation Machine (PAM) into Bayesian tensor decomposition. In the LDA context, top-level groups would contain common genetic mutations, while lower levels apply to more specific subsets.

We first define a non-hierarchical conditional Bayesian Tucker decomposition as multiple features for a given patient. Then, we develop hierarchical topic models for multiple modes. We use these topic models in the conditional Bayesian Tucker decomposition to group patients based on their counts of genetic variants and pathways. Then, we examine and compare the groups generated by different models.

Our contributions are as follows.

\begin{itemize}
\item We define a new formulation for conditional Bayesian Tucker decomposition. 

\item We develop methods for incorporating hierarchical topic models into the Bayesian Tucker decomposition.

\item We study three known properties of CRP in the context of multimodal topic models.

\item We derive a collapsed Gibbs sampler for the conditional Bayesian Tucker decomposition with an arbitrary number of modes.

\item We present a novel method for feature reduction for sparse counting data.
\end{itemize}

In Section \ref{Literature Review} we review existing literature related to our research. In Section \ref{Models} we define our decomposition and topic models. In Section \ref{Properties} we discuss which properties apply to our models. In Section \ref{Algorithms} we discuss algorithms used to compute our decompositions. In Section \ref{Model Evaluation} we evaluate our models trained on the example data sets.

\section{Related Work}
\label{Literature Review}

Two tensor decomposition methods are considered generalizations of singular value decomposition: Tucker decomposition, which decomposes a tensor as a core tensor multiplied by matrices along each mode,
and CANDECOMP/PARAFAC (CP) decomposition, which decomposes a tensor as the sum of rank-one tensors
\citep{kolda2009tensor}. \cite{luo2017tensor,luo2017tensor2} suggested using tensor decompositions to identify latent groups in biomedical fields, including genotyping and phenotyping.

Models exist for decomposing probability tensors using a Dirichlet prior. \cite{dunson_xing_2009, zhou_bhattacharya_herring_dunson_2015} proposed using a Bayesian model to decompose a joint probability tensor using CP decomposition.
Also, \cite{yang_dunson_2016} proposed a Bayesian model to decompose a conditional probability tensor according to the Tucker decomposition.
They used a finite-sized core tensor, which can be adjusted as part of the posterior algorithm. \cite{dunson_xing_2009,yang_dunson_2016,zhou_bhattacharya_herring_dunson_2015} imposed Dirichlet priors (or a Dirichlet stick-breaking prior) on the components of the decomposition and proposed a Gibbs sampler for the posterior computation.
\cite{dunson_xing_2009, zhou_bhattacharya_herring_dunson_2015} used a CP decomposition, while we use a Tucker decomposition. \cite{yang_dunson_2016} used a Tucker decomposition with $1$ response variable conditional on $p$ predictor variables, while we use $p$ feature variables conditional on $1$ sample variable. In addition to this different formulation, we add a hierarchical structure that has not been studied before in the tensor context.

If a conditional Bayesian tensor decomposition consists of only two modes (one sample variable and one feature variable), then both \citepos{yang_dunson_2016} and our model are equivalent. In this case, both decompositions simplify to a matrix factorization method, LDA, which has been studied extensively.
\cite{NIPS2010_3902} and \cite{Buntine2002} stated that LDA can be viewed as a probabilistic matrix factorization of the counting matrix of words in each document into a matrix of topic weights and a dictionary of topics.
Also, \cite{schein_zhou_blei_wallach_2016} noted a connection between Poisson matrix factorization and LDA.

Some researchers studied Bayesian decompositions with other priors. \cite{schein_zhou_blei_wallach_2016} modeled country-to-country interactions using a four-mode tensor to represent an action performed between two countries in a given month. They applied a Bayesian Poisson Tucker decomposition to group countries, actions, and time steps.
\cite{ICML2012Xu_543} proposed a model for computing the Tucker decomposition of a tensor using a normal prior and a variational expectation maximization posterior algorithm. \cite{chi2012tensors} developed a Poisson CP decomposition, called CP Alternating Poisson Regression, and fit their model using a log-likelihood score.

Hierarchical decompositions were studied in a non-Bayesian context. \cite{Hackbusch2009} defined the hierarchical Tucker format, with hierarchy according to vector spaces and subspaces.
\cite{grasedyck_2010} developed algorithms for computing decompositions in the hierarchical Tucker format based on hierarchical singular value decomposition.
\cite{pmlr-v28-song13} defined a recursive decomposition algorithm for estimating a latent tree graphical model of hierarchical tensor decompositions. Their model depicts the joint probability of a set of observed variables as nodes, dependent on their hidden parents.
\cite{Schifanella:2014:MTD:2630935.2532169} proposed a method for the hierarchical decomposition of tensors by adjusting the resolution or size of the core tensor to provide different resolution decompositions of the same data.
Unlike our model, none of these models used Bayesian inference nor expressed the hierarchy of latent topics in each mode. Instead, these models depicted a hierarchy of vector spaces, a hierarchy of hidden variables, or different resolutions.
Also, \cite{teh2006hierarchical} developed a hierarchical Dirichlet process to cluster grouped data.

Another common genetic data analysis technique is to make predictions using a curated list of genes. This method can pick out specific gene markers but cannot find groups of similar genes that contribute to a disease's development. Some studies predicted ASD diagnosis using this method and gene expression rather than genetic sequencing, data \citep{kong2012, Duda2018}. \cite{Abruzzo2015} used six different categories of biomarkers to diagnose ASD in children. They demonstrated that using several biomarkers to diagnose ASD is better than using only one. \cite{bioinformatics/btg382, gsea} do not generate groups of genes, but test if a priori sets are differentially expressed and therefore related to clinical outcome. \cite{Gill2020.11.30.403907} proposed a Bayesian CP decomposition for gene expression data using individuals, genes, tissues, and time as the modes. \cite{LIU2022103958} used a Bayesian Gaussian CP decomposition on gene expression data to identify genomic factors behind breast cancer. \cite{fnins.2023.1212218} used a Bayesian Graph-Laplacian CP decomposition to combine brain magnetic resonance imaging and gene expression data to identify Alzheimer's biomarkers. \cite{journal.pcbi.1012287} created a kernel Bayesian logistic tensor decomposition to find associations between microRNA and diseases.

Our hierarchical Bayesian Tucker decomposition model differs from previous Bayesian tensor decomposition models because we use a unique formulation and impose a hierarchical structure on the latent topics.

\section{Models}
\label{Models}

In what follows, we use bold font to denote vectors, matrices, and tensors and non-bold to denote scalars. If $\boldsymbol u$ is a vector, we denote $u_j$ as its $j^{th}$ component. Table \ref{Variables} summarizes the variables used and their corresponding definitions for ASD and cancer examples.

In Section \ref{Conditional BTD}, we define our conditional non-hierarchical Bayesian Tucker decomposition. In Section \ref{Overview of Hierarchical Topic Models}, we discuss hierarchical topic models for a single mode. In Section \ref{Conditional HBTD}, we define the conditional hierarchical Bayesian Tucker (HBT) decomposition. In Appendix \ref{Generalizations}, we discuss generalizations of the hierarchical topic models to $p\ge 3$ feature modes. More modes can encompass more complex topic hierarchies. We show how to extend our hierarchical models to account for additional modes.

\begin{table}[t]
\caption{Notation}
\label{Variables}
\centering
\small
\begin{tabular}{c|Y{50mm}|Y{50mm}}
 & Definition & Example \\ \hline
$x$ & sample variable id & patient index \\
$\boldsymbol y$ & feature variables & genes and pathways \\
$\boldsymbol z$ & hidden topic & gene and pathway group \\
$\boldsymbol\pi$ & conditional probability tensor $P(\boldsymbol y | x)$ & patients' gene and pathway prevalence \\
$\boldsymbol\phi$ & topic prevalence for each sample & each patient's topic distribution \\
$\boldsymbol\psi$ & prevalence of each feature in each topic & prevalence of genes and pathways in their respective topics \\
$\lambda_x$ & total count for sample $x$ & total genetic variants for patient $x$ \\
$c_j$ & specific value of $\begin{cases} x, & \text{if } j=0 \\ y_j, & \text{if } j>0\end{cases}$ & a specific patient, gene, or pathway \\
$d_j$ & size of tensor in mode $j$ & number of patients, genes, or pathways \\
$p$ & number of feature modes & $2$ modes (genes and pathways) \\
$K_j$ & number of topics for mode $j$ & number of gene or pathway topics \\
$\mathcal{K}$ & set of possible topics & set of gene and pathway topics \\
$\boldsymbol h$ & specific topic set & specific gene and pathway topics \\
$\boldsymbol T_x$ & sample $x$'s path through hierarchical model & path of patient $x$ through PAM \\
$L$ & depth of hierarchical model & 3-level PAM \\
$\ell$ & level in hierarchical model & level of PAM \\
$\tau^{(\ell,j)}$ & number of topics in level $\ell$ for mode $j$ & number of topics in specific level of PAM \\
$\boldsymbol\alpha, \boldsymbol\beta$ & priors for $\boldsymbol\phi, \boldsymbol\psi$ & uniform prior with value 1 \\
$\boldsymbol\gamma$ & prior or parameter for $\boldsymbol T$ & uniform prior over topics in next level \\
\end{tabular}
\end{table}

\subsection{Conditional Bayesian Tucker Decomposition}
\label{Conditional BTD}

First, we define our tensor decomposition for a counting tensor without a hierarchical structure. In the context of applications to ASD and cancer, this model depicts each patient as a mixture of genetic variant and pathway groups, the mixture of genetic variants in each genetic variant group, and the mixture of pathways in pathway groups. These groups can then be analyzed to find patterns and markers of risk factors.

Given a counting tensor $\mathcal{B}=\{b_{c_0\cdots c_p}\}$,\footnote{Convention dictates that if $\mathcal{B}$ is a tensor, $b_{c_0\cdots c_p}$ are its elements, for $c_i\in\{1,\cdots,d_i\}$ and $i\in\{1,\cdots,p\}$.} we first normalize it by dividing by $\lambda_{c_0}=\sum\limits_{c_1=1}^{d_1} \cdots \sum\limits_{c_p=1}^{d_p} b_{c_0\cdots c_p}$ to obtain the $d_0 \times \cdots \times d_p$ conditional probability tensor $\pi_{c_0 \cdots c_p} =\frac{b_{c_0\cdots c_p}}{\lambda_{c_0}} =P(y_{1}=c_1,\cdots,y_{p}=c_p\ |\ x=c_0)$, where $\boldsymbol{y}$ and $x$ are the feature and sample variables. In the genetic applications, $\boldsymbol\pi=P(\text{genetic variants},\ \text{pathways}\ |\ \text{patient})$, $\boldsymbol{y}$ represents the genes and pathways, and $x$ is a patient. We define the conditional Bayesian Tucker decomposition as the Tucker decomposition of the conditional probability tensor,
$\pi_{c_0 \cdots c_p}=\sum\limits_{\boldsymbol h\in \K}\phi_{c_0 \boldsymbol h} \prod\limits_{j=1}^{p}\psi_{h_j c_j}^{(j)}$,
with $\K=\left\{(h_1,\cdots,h_p)\ |\ h_j \in \{1,\cdots,K_j\}\ \forall\ j\in \{1,\cdots,p\}\right\}$, latent classes $\boldsymbol z$, $\phi_{c_0h_1\cdots h_p}=P(z_{1}=h_1,\cdots,z_{p}=h_p\ |\ x=c_0)$, and $\psi_{h_jc_j}^{(j)}=P(y_{j}=c_j\ |\ z_{j}=h_j)$ for all $j$. In our example,
$\boldsymbol\phi =P(\text{gene groups},$ pathway groups $|\ \text{patient})$,
$\boldsymbol\psi^{(1)} =P(\text{genes}$ $|\ \text{gene groups})$, and
$\boldsymbol\psi^{(2)} =P(\text{pathways}$ $|\ \text{pathway groups})$.
In this model, for each $x$, a joint topic distribution over topic vectors $\boldsymbol h \in \K$ is first selected, governed by core tensor $\boldsymbol\phi$. Next, for all modes $j>0$, the selected topic $h_j$ is a mixture of choices $\{1,\cdots,d_j\}$, governed by auxiliary matrix $\boldsymbol\psi^{(j)}$. We note that $\sum\limits_{h\in \K}\phi_{c_0\boldsymbol h}=1$ for all $c_0\in\{1,\cdots,d_0\}$ and $\sum\limits_{c_j=1}^{d_j}\psi_{h_j c_j}^{(j)}=1$ for all $j$ and $h_j$.

For ease of notation, we define $K=\prod\limits_{i=1}^p K_i$ and map $\text{vec}:\K\mapsto\{1,\cdots, K\}$ as a one-to-one mapping from a tuple of topics to a single topic index. Notation $\text{vec}^{-1}$ is the inverse map. Our model does not depend on the choice of such a map.\footnote{An example mapping would be $\text{vec}(\boldsymbol k)=k_1+(k_2-1)K_1+\cdots+(k_p-1)\prod\limits_{i=1}^{p-1}K_i$. This is a generalization of the column-major order map.} The generative process is presented in Algorithm \ref{Generative Process}, where $\S_d=\left\{\boldsymbol v \in \R^d_+\ \Big|\ \sum\limits_{i=1}^dv_i=1\right\}$ is the $d$-dimensional probability simplex. Our use of Dirichlet and multinomial distributions is consistent with the model and data.

\begin{algorithm}[t]
\small
\begin{multicols}{2}
\For{$x=1,\cdots,d_0$}{
Draw core tensor $\tilde{\boldsymbol\phi}_x\sim\text{Dir}(\boldsymbol\alpha)\in \S_K$ \\
\For{$\boldsymbol k\in\K$}{
$\phi_{x\boldsymbol k}=\tilde{\phi}_{x\text{vec}(\boldsymbol k)}$ \\
}
}
\For{$j=1,\cdots,p$}{
\For{$k=1,\cdots,K_j$}{
Draw auxiliary matrices $\boldsymbol\psi_{k}^{(j)}\sim\text{Dir}\left(\boldsymbol\beta^{(j)}\right)\in \S_{d_j}$ \\
}
}
\For{$x=1,\cdots,d_0$}{
\For{$i=1,\cdots,\lambda_x$}{
Draw latent topics $\varepsilon\sim\text{Mult}\left(\{1,\cdots,K\},\tilde{\boldsymbol\phi}_x\right)$ \\
$\boldsymbol z_{i}^{(x)}=\text{vec}^{-1}(\varepsilon)$ \\
\For{$j=1,\cdots,p$}{
Draw features $y_{ij}^{(x)}\sim\text{Mult}\left(\{1,\cdots,d_j\},\psi_{z_{ij}^{(x)}}^{(j)}\right)$ \\
}
}
}
\end{multicols}
\caption{Generative Process}
\label{Generative Process}
\end{algorithm}

The model probability is the product of components: factor matrices given their priors, core tensor given its prior, and individual count probabilities. The probability of the factor matrices $\boldsymbol\psi$ given its priors $\boldsymbol\beta$ is a nested product over modes $j$ and topics $h_j$ in that mode given the prior for that mode:
\begin{equation} P(\boldsymbol\psi|\boldsymbol\beta) = \prod\limits_{j=1}^p \prod\limits_{h_j=1}^{K_j}P\left(\boldsymbol\psi_{h_j}^{(j)}\big|\boldsymbol\beta^{(j)}\right).\label{feature matrices}\end{equation}
The probability of the core tensor $\boldsymbol\phi$ given its prior $\boldsymbol\alpha$ is a product over all samples $x$ given the prior:
\begin{equation} P(\boldsymbol\phi|\boldsymbol\alpha) = \prod\limits_{x=1}^{d_0} P\left(\boldsymbol\phi_{x}|\boldsymbol\alpha\right).\label{core tensor}\end{equation}
The individual count probabilities are a nested product over samples $x$ and counts within each sample $\lambda_x$ of the hidden topics $\boldsymbol z_{i}^{(x)}$'s probability given the core tensor $\boldsymbol\psi_x$ and the feature variable $\boldsymbol y_{i}^{(x)}$'s probability given the feature matrices $\boldsymbol\psi_{\boldsymbol z_{i}^{(x)}}$:
\begin{equation} P(\boldsymbol Y,\boldsymbol Z |\boldsymbol\phi, \boldsymbol\psi) = \prod\limits_{x=1}^{d_0} \prod\limits_{i=1}^{\lambda_{x}} P\left(\boldsymbol y_{i}^{(x)}\big|\boldsymbol\psi_{\boldsymbol z_{i}^{(x)}}\right)P\left(\boldsymbol z_{i}^{(x)}\big|\boldsymbol\phi_{x}\right).\label{individual counts}\end{equation}
The overall model probability is a product of the components in \refp{feature matrices}-\refp{individual counts}:
\begin{equation}
P(\boldsymbol Y, \boldsymbol Z, \boldsymbol\phi, \boldsymbol\psi|\boldsymbol\alpha, \boldsymbol\beta)= P(\boldsymbol\psi|\boldsymbol\beta)P(\boldsymbol\phi|\boldsymbol\alpha) P(\boldsymbol Y,\boldsymbol Z |\boldsymbol\phi, \boldsymbol\psi).\label{non-hierarchical prob}
\end{equation}

\subsection{Overview of CRP \& PAM}
\label{Overview of Hierarchical Topic Models}

As a prelude to our discussion on hierarchical topic models in multiple modes, we first discuss relevant hierarchical topic models for a single feature mode. We start with the nested Chinese Restaurant Process (nCRP) and hierarchical LDA (hLDA). Then, we define the Pachinko Allocation Machine (PAM) and hierarchical PAM (hPAM).

\cite{NIPS2003_2466} defined nCRP as follows. First, each patient is assigned a root topic. Then, each patient is assigned a topic on the next level of the tree using a CRP. This CRP is only used by patients from the same parent topic, and its topics are distinct from other parent topics. This process repeats until each patient has drawn a path of $L$ topics, creating a tree of depth $L$. The transition from LDA to hLDA is made by drawing a path $\boldsymbol T_x$ through the nCRP for each patient $x$, then using each patient's path as their topics set (only these topics have non-zero prevalence).

\cite{Li:2006:PAD:1143844.1143917} defined PAM as a model that connects topics with a directed acyclic graph (DAG). PAM samples a topic path through the DAG, starting at the root and sampling each child according to the multinomial distribution of the current topic. Although a PAM can use an arbitrary DAG, we use a DAG with a leveled structure, i.e., for a set number of levels $L$, each node on level $\ell \in \{1,\cdots, L-1\}$ is a parent of every node on the next level $\ell+1$. In this model, the number of topics $\tau^{(\ell)}$ are predetermined for each level $\ell$. \cite{Mimno2007} defined hPAM as a PAM model where all nodes are associated with distributions over the vocabulary rather than only those on the lowest level. In this model, a path is drawn through the PAM for each patient, where the set of nodes visited determines the patient's topics set.

\subsection{Conditional Hierarchical Bayesian Tucker Decomposition}
\label{Conditional HBTD}

Next, we extend our conditional Bayesian Tucker decomposition to topic hierarchies. For simplicity, we assume $p=2$ feature modes, but we discuss generalizations to $p\ge3$ in Appendix \ref{Generalizations}.

With genes and pathways, some variants are more common than others. Thus, it is helpful to organize topics hierarchically where more common variants determine higher-level groups, while rarer variants determine specific subgroups.

As outlined in Section \ref{Overview of Hierarchical Topic Models}, the transition from LDA to hLDA uses an nCRP. The challenge with implementing this in our context is that we have multiple modes of topics. It is unclear how to generalize nCRP, with each table representing a topic pair rather than a single topic. In the case of hLDA, a customer sitting at a new table represents drawing a new topic. But if each table represents a topic pair rather than a single topic, what does a new table represent? A new topic in one or both modes or a new combination of existing topics? There is no clear way of determining what a new table represents without imposing an order on pairs of groups. We prove in Section \ref{Conditions} that there is no natural order where the order of customers does not matter.

We describe two solutions to this problem: the independent topic model and the hierarchical topic model. The hierarchical model requires an order of modes since the choice of topic for each mode depends on the topic in its parent mode. This method allows us to incorporate gene and pathway topics into one hierarchical model. Given $x\in\{1,\cdots,d_0\}$, let $\boldsymbol T_x$ be its path through a conceptual topic model. The topic distribution $\boldsymbol\phi_x$ along any such $\boldsymbol T_x$ is drawn from $\text{Dir}(\boldsymbol\alpha)\in\S^L$, with $L$ being the length of $\boldsymbol T_x$. In other words, $\boldsymbol T_x$ dictates which topics have positive probabilities.

\begin{figure}[t]
\centering
\begin{tikzpicture}[scale=.9, transform shape]
\tikzstyle{main}=[circle, minimum size = 10mm, thick, draw =black!80, node distance = 12mm]
\tikzstyle{connect}=[-latex, thick]
\tikzstyle{box}=[rectangle, draw=black!100]
  \node[main, fill = white!100] (alpha) {$\alpha$};
  \node[main] (phi) [right=of alpha] {$\phi_{c_0}$};
  \node[main] (t) [below=6mm of phi] {$T_{c_0}$};
  \node[main] (gamma) [left=of t] {$\gamma$};
  \node[main] (z) [right=of phi] {Z};
  \node[main, fill = black!10] (y) [right=of z] {Y};
  \node[main] (psi) [right=of y] {$\psi^{(j)}$};
  \node[main] (beta) [right=of psi] {$\beta$};
  \path (alpha) edge [connect] (phi)
        (phi) edge [connect] (z)
		(z) edge [connect] (y)
		(beta) edge [connect] (psi)
        (psi) edge [connect] (y)
        (t) edge [connect] (phi)
        (gamma) edge [connect] (t);
  \node[rectangle, inner sep=2.8mm,draw=black!100, fit= (psi)] (box1){};
  \node[above left=-.8mm and -.8mm] at (box1.south east) {$K_j$};
  \node[rectangle, inner sep=4mm,draw=black!100, fit= {(psi) ($(psi.east)+(3mm,0)$)}] (box2) {};
  \node[above left=-.8mm and -.8mm] at (box2.south east) {$p$};
  \node[rectangle, inner sep=3mm,draw=black!100, fit= (z) (y)] (box3) {};
  \node[above left=-.8mm and -.8mm] at (box3.south east) {$\lambda_{c_0}$};
  \node[rectangle, inner sep=4mm, draw=black!100, fit = {(phi) (z) (y) ($(t.south)+(0,1.5mm)$)}] (box4) {};
  \node[above left=-.8mm and -.8mm] at (box4.south east) {$d_0$};
\end{tikzpicture}
\caption{Plate Diagram for a Conditional Hierarchical Bayesian Tucker Decomposition}
\label{Plate Diagram}
\end{figure}

The probability for an HBT decomposition (illustrated in its plate diagram, Figure \ref{Plate Diagram}) is constructed similarly to that of the non-hierarchical model \refp{non-hierarchical prob}. While the factor matrices \refp{feature matrices} and individual count probabilities \refp{individual counts} are the same, the core tensor probability \refp{core tensor} must be modified to incorporate the hierarchical model. Here, we incorporate the probability of each sample's $x$ path $\boldsymbol T_x$ given the parameter $\boldsymbol\gamma$. The core tensor $\boldsymbol\phi_{x}$'s probability for each sample depends on both prior $\boldsymbol\alpha$ and path $\boldsymbol T_x$ since the path dictates which topics have non-zero probabilities:
\begin{equation}
P(\boldsymbol\phi,\boldsymbol T|\boldsymbol\alpha, \boldsymbol\gamma)=\prod\limits_{x=1}^{d_0}P\left(\boldsymbol\phi_{x}|\boldsymbol\alpha, \boldsymbol T_{x}\right) P\left(\boldsymbol T_{x}|\boldsymbol\gamma\right).
\label{hierarchical core}
\end{equation}
Combining \refp{feature matrices}, \refp{individual counts}, and \refp{hierarchical core} yields the overall model probability:
\begin{equation}
P(\boldsymbol Y, \boldsymbol Z, \boldsymbol\phi, \boldsymbol\psi, \boldsymbol T|\boldsymbol\alpha, \boldsymbol\beta,\boldsymbol\gamma)=P(\boldsymbol\psi|\boldsymbol\beta)P(\boldsymbol\phi,\boldsymbol T|\boldsymbol\alpha, \boldsymbol\gamma) P(\boldsymbol Y,\boldsymbol Z |\boldsymbol\phi, \boldsymbol\psi).
\label{hierarchical prob}
\end{equation}

\begin{figure}[t]
\centering
\subfigure[Independent Trees]{
\centering
\label{Independent Trees}
\hspace{1mm}
\begin{tikzpicture}[scale=.8, transform shape]
\tikzstyle{main}=[circle, minimum size =3mm, thick, draw =black!80, node distance = 4mm]
\tikzstyle{connect}=[-latex, ultra thin]
\tikzstyle{connectB}=[-latex, ultra thick]
\tikzstyle{rect}=[rectangle, minimum size = 3mm, thick, draw =black!80, node distance = 4mm, minimum width = 12mm]
  \node[main, fill = black!20] (1) {};
  \node[main, fill = black!20] (22) [below right=4mm and 4mm of 1] {};
  \node[main] (21) [left=14mm of 22] {};
  \node[main] (23) [right=9mm of 22] {};
  \node[main] (32) [below=of 21] {};
  \node[main] (31) [left=of 32] {};
  \node[main] (33) [right=of 32] {};
  \node[main, fill = black!20] (34) [right=of 33] {};
  \node[main] (35) [right=of 34] {};
  \node[main] (36) [right=of 35] {};
  \path (1) edge [connect] (21)
 (1) edge [connectB] (22)
 (1) edge [connect] (23)
 (21) edge [connect] (31)
 (21) edge [connect] (32)
 (21) edge [connect] (33)
 (22) edge [connectB] (34)
 (22) edge [connect] (35)
 (23) edge [connect] (36);
 \node[above=2 mm of 1] {\small Mode 1};
 \node[right=20 mm of 1] (L1) {\small Lvl 1};
 \node[below=2 mm of L1] (L2) {\small Lvl 2};
 \node[below=2 mm of L2] (L3) {\small Lvl 3};
 \node[below=2 mm of L3] (L4) {\small Lvl 4};
 \node[main, fill = black!20] (201) [right=12mm of L1] {};
 \node[main] (222) [below right=4mm and 4.5mm of 201] {};
 \node[main, fill = black!20] (221) [left=11mm of 222] {};
 \node[main, fill = black!20] (232) [below right=4mm and 1.5mm of 221] {};
 \node[main] (231) [left=of 232] {};
 \node[main] (233) [right=of 232] {};
 \node[main] (234) [right=of 233] {};
 \node[main] (241) [below=of 231] {};
 \node[main, fill = black!20] (242) [below=of 232] {};
 \node[main] (243) [below=of 233] {};
 \node[main] (244) [below=of 234] {};
 \path (201) edge [connectB] (221)
 (201) edge [connect] (222)
 (221) edge [connect] (231)
 (221) edge [connectB] (232)
 (222) edge [connect] (233)
 (222) edge [connect] (234)
 (231) edge [connect] (241)
 (232) edge [connectB] (242)
 (233) edge [connect] (243)
 (234) edge [connect] (244);
 \node[above=2 mm of 201] {\small Mode 2};
\end{tikzpicture}
\hspace{1mm}}%
\subfigure[PAM-Based]{
\centering
\label{PAM}
\hspace{1mm}
\begin{tikzpicture}[scale=.8, transform shape]
\tikzstyle{main}=[circle, minimum size =3mm, thick, draw =black!80, node distance = 4mm]
\tikzstyle{connect}=[-latex, ultra thin]
\tikzstyle{connectB}=[-latex, ultra thick]
\tikzstyle{rect}=[rectangle, minimum size = 3mm, thick, draw =black!80, node distance = 4mm, minimum width = 12mm]
  \node[main, fill = black!20] (1) {};
  \node[main, fill = black!20] (22) [below= of 1] {};
  \node[main] (21) [left=of 22] {};
  \node[main] (23) [right=of 22] {};
  \node[main] (33) [below=of 22] {};
  \node[main, fill = black!20] (32) [left=of 33] {};
  \node[main] (31) [left=of 32] {};
  \node[main] (34) [right=of 33] {};
  \node[main] (35) [right=of 34] {};
  \node[main] (41) [below=of 31] {};
  \node[main] (42) [below=of 32] {};
  \node[main] (43) [below=of 33] {};
  \node[main, fill = black!20] (44) [below=of 34] {};
  \node[main] (45) [below=of 35] {};
  \path (1) edge [connect] (21)
 (1) edge [connectB] (22)
 (1) edge [connect] (23)
 (21) edge [connect] (31)
 (21) edge [connect] (32)
 (21) edge [connect] (33)
 (21) edge [connect] (34)
 (21) edge [connect] (35)
 (22) edge [connect] (31)
 (22) edge [connectB] (32)
 (22) edge [connect] (33)
 (22) edge [connect] (34)
 (22) edge [connect] (35)
 (23) edge [connect] (31)
 (23) edge [connect] (32)
 (23) edge [connect] (33)
 (23) edge [connect] (34)
 (23) edge [connect] (35)
 (31) edge [connect] (41)
 (31) edge [connect] (42)
 (31) edge [connect] (43)
 (31) edge [connect] (44)
 (31) edge [connect] (45)
 (32) edge [connect] (41)
 (32) edge [connect] (42)
 (32) edge [connect] (43)
 (32) edge [connectB] (44)
 (32) edge [connect] (45)
 (33) edge [connect] (41)
 (33) edge [connect] (42)
 (33) edge [connect] (43)
 (33) edge [connect] (44)
 (33) edge [connect] (45)
 (34) edge [connect] (41)
 (34) edge [connect] (42)
 (34) edge [connect] (43)
 (34) edge [connect] (44)
 (34) edge [connect] (45)
 (35) edge [connect] (41)
 (35) edge [connect] (42)
 (35) edge [connect] (43)
 (35) edge [connect] (44)
 (35) edge [connect] (45);
 \node[right=20 mm of 1] (M1L1) {\small Mode 1, Lvl 1};
 \node[below=2 mm of M1L1] (M2L1) {\small Mode 2, Lvl 1};
 \node[below=2 mm of M2L1] (M1L2) {\small Mode 1, Lvl 2};
 \node[below=2 mm of M1L2] (M2L2) {\small Mode 2, Lvl 2};
\end{tikzpicture}
\hspace{1mm}}
\caption{Example Hierarchical Topic Models ($p=2$)}
\end{figure}

We define an independent topic model ($p=2$) as containing two separate topic models for each mode (example shown in Figure \ref{Independent Trees}). Each customer $x$ draws two paths, one for each mode, $\boldsymbol T_x^{(1)}$ and $\boldsymbol T_x^{(2)}$, represented as a set of topics in that mode. The overall topic list $\boldsymbol T_x$ consists of all possible pairs in $\boldsymbol T_x^{(1)}$ and $\boldsymbol T_x^{(2)}$. A common choice draws these paths $\boldsymbol T_x^{(1)} \text{ and } \boldsymbol T_x^{(2)}$ using two independent nCRPs as follows: Let $\boldsymbol P^{(x, \ell, m)}$ be the probabilities associated with two independent CRPs, with hyper-parameters $\gamma_m^{(\ell)}$, where $m\in\{1,2\}$ is the mode and $\ell\in\{1,\cdots, L_m\}$ is the level of the tree. In this model, the number of tables (topics) in each mode $\tau^{(\ell, m)}$ varies. Specifically, 
$$P_i^{(x, \ell, m)} =\begin{cases} \frac{n_i}{\gamma_m^{(\ell)}+n-1}, & \text{table } i \text{ occupied} \\ \frac{\gamma_m^{(\ell)}}{\gamma_m^{(\ell)}+n-1}, & \text{next unoccupied table } i \end{cases},$$
where customer $x$ is the $n^{th}$ customer at the restaurant and $n_i$ customers are already seated at table $i$. While others have imposed independent hierarchical structures on the topics of different modes \citep{pmlr-v28-song13, Schifanella:2014:MTD:2630935.2532169}, none have incorporated nCRP into their topic models.

We define a hierarchical topic model ($p=2$) as a single topic model with a dominant mode (example depicted in Figure \ref{PAM}). Without loss of generality, we assume the first mode is dominant. At the first level, each customer $x$ starts at the root topic $1$ in mode $1$; then, each customer chooses a topic $j$ in mode $2$, according to probability $\boldsymbol P_1^{(1,1)}$. At each subsequent level $\ell\in\{2,\cdots,L\}$, each customer chooses a topic $i$ in mode $1$, according to $\boldsymbol P_j^{(\ell-1,2)}$, then a topic $j$ in mode $2$, according to $\boldsymbol P_i^{(\ell,1)}$. Here, $\boldsymbol P_i^{(\ell,m)}$ is the probability distribution over its children topics. There are two obvious choices for the overall topic list $\boldsymbol T_x$, the pairs of topics visited at each level of the DAG (which we call the level method) or all possible pairs of elements in the topic lists for each mode (we call this the Cartesian method). If PAM ideas are used, $\boldsymbol P_i^{(\ell, m)}$ (for all $i\in\{1,\cdots,\tau^{(\ell,1)}\}$ and $m \in\{1,2\}\}$) are multinomials drawn from Dirichlet distributions, i.e., $\boldsymbol P_i^{(\ell,m)}\sim \text{Dir}\left(\boldsymbol\gamma_i^{(\ell,m)}\right)\in \S_{\tau^{(\ell,m)}}$. While PAM has been used to model interactions between variables for LDA \citep{Li:2006:PAD:1143844.1143917}, it has not been used to model topic interactions between multiple modes.

\section{Properties}
\label{Properties}

First, in Section \ref{Definitions} we define the exchangeability, partition, and rich-get-richer properties based on those of the CRP, which govern topic models. \cite{Blei2007} defined the exchangeability property as the model probability being invariant under the permutations of customers. The partition property implies that two arrangements of the same number of topics with equal occupants have the same probability. In the single-mode case, this property follows from exchangeability, but this is not necessarily true for multiple modes. \cite{Teh2010a} described the rich-get-richer property as large clusters growing faster than smaller clusters. For simplicity, we assume $p=2$ feature modes and study a single level. In Section \ref{Conditions}, we prove that these properties cannot strictly apply in the $p=2$ case.

\subsection{Definitions}
\label{Definitions}

To define these properties, we first need to define a partition, an arrangement of customers into topics. Here, we assume the number of customers $n$ is fixed.

\begin{definition} The count of customers assigned to each topic is a \textbf{partition}.
For partition $\boldsymbol\rho =\{\rho_{ij}\}$, $\rho_{ij}$ is the count of customers assigned topic $i$ in mode 1 and topic $j$ in mode 2.
\end{definition}

\noindent In other words, $\rho_{ij}$ is the number of customers sitting at the same table. Here, $\boldsymbol\rho$ is a matrix, with rows and columns representing the topics in each mode. Note that $n=\sum\limits_{i,j}\rho_{ij}$.

We first define the exchangeability property. This property is necessary because if the order of customers affects the model, it would be impossible to determine how well it would perform inference, i.e., predict new customers. 

\begin{definition} We say that a topic model has the \textbf{exchangeability} property if the probability of a partition does not depend on the order of the customers, i.e., $P(\boldsymbol\rho|\text{order of customers}) = P(\boldsymbol\rho|\Pi(\text{order of customers}))$, where $\Pi$ is a permutation of the order.
\end{definition}

Next, we define two variants of the partition property: strict and loose. This property implies that the label and order of topics do not matter.

\begin{definition} We say that a topic model has the \textbf{strict partition} property if for all partitions $\boldsymbol\rho$ and $\Pi(\boldsymbol\rho)$ permutations of the elements of $\boldsymbol\rho$, $P(\boldsymbol\rho)=P(\Pi(\boldsymbol\rho))$.
\end{definition}

\begin{definition} We say that a topic model has the \textbf{loose partition} property if for all partitions $\boldsymbol\rho$ and $\Pi(\boldsymbol\rho)$ permutations of the rows and columns of $\boldsymbol\rho$, $P(\boldsymbol\rho)=P(\Pi(\boldsymbol\rho))$.
\end{definition}

\noindent Note that if strict partition holds, so does loose partition. However, the converse does not necessarily hold.

Lastly, we define the rich-get-richer property under the assumption of exchangeability and partition properties. The general idea is that new customers tend to join topics with more customers. Let $\rho_{(\cdot)i}=\sum\limits_{j=1}^{K_1}\rho_{ji}$ and $\rho_{i(\cdot)}=\sum\limits_{j=1}^{K_2}\rho_{ij}$. In addition, let $\xi_i$ be the probability of a new customer being assigned to topic $i$ in mode 1 and $\theta_{ij}$ be the probability of a new customer being assigned to topic $j$ in mode 2 given that the customer was assigned to topic $i$ in mode 1. Due to the exchangeability property, $\xi_i$ and $\theta_{ij}$ do not depend on the customers' order. Also, for the independent topic model, we can drop the dependency in $\theta$ on the topic in the first mode and denote $\theta_{i}$ as the probability of a new customer being assigned to topic $i$ in mode 2.

\begin{definition} We say that an independent topic model has the \textbf{rich-get-richer} property if
\begin{enumerate}
\item assuming $\rho_{i(\cdot)}\neq0$ and $\rho_{j(\cdot)}\neq0$, then
$\xi_i > \xi_j$ if and only if $\rho_{i(\cdot)}>\rho_{j(\cdot)}$ (Mode 1), and
\item assuming $\rho_{(\cdot) i}\neq0$ and $\rho_{(\cdot) j}\neq0$, then $\theta_{i}>\theta_{j}$ if and only if $\rho_{(\cdot)i}>\rho_{(\cdot)j}$ (Mode 2).
\end{enumerate}
\label{iRGR}
\end{definition}

\begin{definition} We say that a hierarchical topic model has the \textbf{rich-get-richer} property if 
\begin{enumerate}
\item assuming $\rho_{i(\cdot)}\neq 0$ and $\rho_{j(\cdot)}\neq0$,
then $\xi_i > \xi_j$ if and only if $\rho_{i(\cdot)}>\rho_{j(\cdot)}$ (Mode 1), and
\item assuming $\rho_{ki}\neq 0$ and $\rho_{kj}\neq0$, then $\theta_{ki}>\theta_{kj}$ if and only if $\rho_{ki}>\rho_{kj}$, $\forall\ k$ (Mode 2).
\end{enumerate}
\label{hRGR}
\end{definition}

\noindent The difference between definitions  \ref{iRGR} and \ref{hRGR} is that the independent model only compares the fibers (rows and columns) of the partition. In contrast, the hierarchical model compares individual elements within each fiber.

The rich-get-richer property is important because it guides the model to group similar genes and pathways rather than creating a new topic. A model without this property would place each gene in a separate topic, negating the model’s utility.

\subsection{Conditions}
\label{Conditions}

Now, we determine under which conditions the loose and strong partition properties hold in multi-modal hierarchical models. While we sketch our proofs here, detailed proofs are in Appendix \ref{Proofs}. We begin with the PAM-based model ($p=2$), giving the specific case in which the loose partition property holds.

\begin{theorem}
\label{Loose Partition}
The loose partition property holds in the PAM model ($p=2$) if and only if the parameters of the Dirichlet distributions are symmetric.
\end{theorem}

In the proof, we consider a single Dirichlet distribution and show that a uniform prior is sufficient for the loose partition property to hold. Then, we use a non-uniform example to show that uniformity is necessary for this property to hold. Lastly, we argue that the theorem holds for the entire PAM model if and only if it holds for every node in the model.

Next, we prove that a hierarchical topic model ($p=2$) for which all three properties apply does not exist. First, we show that the probability of a new customer being assigned a topic in each mode has a functional form.

\begin{lemma}
If there exists a hierarchical topic model ($p=2$) where the rich-get-richer strong partition, and exchangeability properties hold, then $\boldsymbol\xi$ and $\boldsymbol\theta$ are of the form:
$\xi_i\propto\begin{cases} f\left(\rho_{i(\cdot)}\right), & \rho_{i(\cdot)}>0 \\ \gamma_{0}(K_1), & \rho_{i(\cdot)}=0\end{cases}$ and $\theta_{ij}\propto\begin{cases} g_i(\rho_{ij}), & \rho_{ij}>0 \\ \gamma_{i}(K_2), & \rho_{ij}=0\end{cases}$,
for some functions $f$ and $g_i$ $(\forall\ i)$.
\label{function form lemma}
\end{lemma}

We show that this formulation is sufficient to satisfy the strong partition property and necessary to satisfy the rich-get-richer and exchangeability properties.

Next, we show that these functions must be linear.

\begin{lemma}
Given the assumptions and results of Lemma \ref{function form lemma}, $f$ and $g_i$ (for all i) are linear.
\label{linear lemma}
\end{lemma}

We use the formula from Lemma \ref{function form lemma} to compute the probabilities for two different orderings. Given the exchangeability property, these formulations must be equal. From this, we conclude that $f$ must be linear.

Finally, we use the linear form to show that such a model does not exist.

\begin{theorem}
\label{DNE}
There does not exist a hierarchical topic model ($p=2$) where the rich-get-richer, strong partition, and exchangeability properties all hold.
\end{theorem}

We exhibit in Appendix \ref{Proofs} an example to show that the formulation in Lemma \ref{linear lemma} violates the strict partition property.

However, the independent trees model allows all three properties. This follows from the properties of each independent CRP.

\begin{theorem}
The exchangeability, loose partition, and rich-get-richer properties hold for the independent trees model.
\end{theorem}

This conclusion rationalizes our choice of model with the loose partition property holding rather than the strong version.

\section{Algorithms}
\label{Algorithms}

We use a Gibbs sampler to compute the posterior model based on our prior and our samples. This algorithm alternates between drawing new paths through the hierarchical topic model and solving the Bayesian Tucker decomposition problem. This scheme is similar to that developed for hLDA by \cite{NIPS2003_2466}. We present a general overview of our Gibbs sampler in Algorithm \ref{Algorithm Overview}.

\begin{algorithm}[ht]
\small
Initialize hierarchical topic model and Bayesian Tucker decomposition \\
\For{$i=1,\cdots,I$}{
Draw Bayesian Tucker decomposition $\boldsymbol{\phi}, \boldsymbol{\psi}$ and latent topics $\boldsymbol{z}$ (Algorithm \ref{Collapsed Bayesian Tucker Decomposition Gibbs Sampler} or \ref{Non-Collapsed Bayesian Tucker Decomposition Gibbs Sampler}) \\
Draw hierarchical topic model $\boldsymbol{T}$ (Algorithm \ref{Independent Trees Algorithm}, \ref{Collapsed PAM-Based Hierarchical Topic Model Algorithm}, or \ref{Non-Collapsed PAM-Based Hierarchical Topic Model Algorithm}) \\
}
\caption{Algorithm Overview}
\label{Algorithm Overview}
\end{algorithm}

In Section \ref{Bayesian Tucker Decomposition}, we present two algorithms for sampling from the Bayesian Tucker decomposition, and in Section \ref{Hierarchical Topic Models}, we present Gibbs samplers for both the independent trees and PAM-based topic models. In Appendix \ref{variations}, we propose variations and modifications that can be made to these models to improve performance.

\subsection{Sampling from Bayesian Tucker Decomposition}
\label{Bayesian Tucker Decomposition}

Here, we present a collapsed Gibbs sampler for our Bayesian Tucker decomposition, Algorithm \ref{Collapsed Bayesian Tucker Decomposition Gibbs Sampler}, which we derive in Appendix \ref{Collapsed Gibbs}. We give a non-collapsed version, Algorithm \ref{Non-Collapsed Bayesian Tucker Decomposition Gibbs Sampler}, in Appendix \ref{Non-Collapsed Algorithms}. When applying a topic model, we constrain these algorithms so that $\phi_{x\boldsymbol k}$ is positive for $\boldsymbol k\in \boldsymbol{T}_x$ and zero for $\boldsymbol k\not\in \boldsymbol{T}_x$, for all $x$.

While these algorithms are based on the \cite{yang_dunson_2016} algorithm, they did not derive a collapsed sampler, and their model uses a different decomposition form. Our collapsed algorithm makes sampling easier by integrating out the Dirichlet distributions. Here, $n_x^k$ is the count of topic $k$ given sample $x$, $m^{(j)}_{h y}$ is the count of feature $y$ in  mode $j$ given topic $h$, and a superscript $-xi$ indicates omitting count $i$. Recall that the bold font implies $\boldsymbol{m}^{(j)}_{h}$ is a vector over $ y$s.

\begin{algorithm}[ht]
\small
\For{$x=1,\cdots,d_0$}{
\For{$i=1,\cdots,\lambda_x$}{
Compute latent topic probabilities $P\left(\boldsymbol z_{i}^{(x)}=\boldsymbol{k}\ \big|\ \boldsymbol{n}^{-xi},\boldsymbol{m}^{-xi},\boldsymbol{\alpha},\boldsymbol{\beta}\right)\propto\left(n_x^{\boldsymbol k,-xi}+\alpha_{\boldsymbol k}\right)\prod\limits_{j=1}^p\frac{m^{(j),-xi}_{k_j y_j}+\beta_{y_j}^{(j)}}{\sum_{y=1}^{d_j}m^{(j),-xi}_{k_j y}+\beta_{y}^{(j)}}$ \\
Draw latent topic $\boldsymbol z_{i}^{(x)}$ from $P\left(\boldsymbol z_{i}^{(x)}=\boldsymbol{k}\ \big|\ \boldsymbol{n}^{-xi},\boldsymbol{m}^{-xi},\boldsymbol{\alpha},\boldsymbol{\beta}\right)$ \\
}
}
\caption{Collapsed Bayesian Tucker Decomposition Gibbs Sampler}
\label{Collapsed Bayesian Tucker Decomposition Gibbs Sampler}
\end{algorithm}

\subsection{Hierarchical Topic Models}
\label{Hierarchical Topic Models}

Next, we present Gibbs samplers for the independent trees and PAM-based hierarchical topic models. For the PAM-based model, we give both non-collapsed and collapsed versions.

Algorithm \ref{Independent Trees Algorithm} describes how paths are drawn through an independent trees model. This algorithm is a generalization of the hLDA algorithm by \cite{NIPS2003_2466}. Here, $m_{c_{x,\ell},y}^{(j),-x}$ is the count of features $y$ assigned to topic $c_{x,\ell}$ in mode $j$, $m_{c_{x,\ell},(\cdot)}^{(j),-x}:=\sum\limits_{y=1}^{d_j} m_{c_{x,\ell},y}^{(j),-x}$, and subscript $-x$ indicates omitting sample $x$.

\begin{algorithm}[ht]
\small
\For{$j=1,\cdots,p$}{
\For{$x=1,\cdots,d_0$}{
\For{$\ell=1,\cdots,L_j$}{
$P\left(\boldsymbol{Y}_x^{(j)}|\boldsymbol{Y}_{-x}^{(j)},\boldsymbol{c}^{(j)},\boldsymbol{Z}^{(j)}\right)=\frac{\Gamma\left(m_{c_{x,\ell},(\cdot)}^{(j),-x}+d_j\beta^{(j)}\right)}{\prod_y\Gamma\left(m_{c_{x,\ell},y}^{(j),-x}+\beta^{(j)}\right)}\frac{\prod_y\Gamma\left(m_{c_{x,\ell},y}^{(j)}+\beta^{(j)}\right)}{\Gamma\left(m_{c_{x,\ell},(\cdot)}^{(j)}+d_j\beta^{(j)}\right)}$ \\
$P\left(c_{x,\ell}^{(j)}|\boldsymbol{c}_{-x\ell}^{(j)}\right)=$ prior imposed by CRP \\
$P\left(c_{x,\ell}^{(j)}|\boldsymbol{Y}^{(j)},\boldsymbol{c}_{-x,\ell}^{(j)},\boldsymbol{Z}^{(j)}\right)\propto P\left(\boldsymbol{Y}_x^{(j)}|\boldsymbol{Y}_{-x}^{(j)},\boldsymbol{c}_\ell^{(j)},\boldsymbol{Z}^{(j)}\right) P\left(c_{x,\ell}^{(j)}|\boldsymbol{c}_{-x\ell}^{(j)}\right)$
}
}
}
\caption{Independent Trees Algorithm}
\label{Independent Trees Algorithm}
\end{algorithm}

We provide collapsed (Algorithm \ref{Collapsed PAM-Based Hierarchical Topic Model Algorithm}) and non-collapsed (Algorithm \ref{Non-Collapsed PAM-Based Hierarchical Topic Model Algorithm} in Appendix \ref{Non-Collapsed Algorithms}) algorithms to draw paths through the PAM-based model. The collapsed algorithm is derived by integrating out the Dirichlet distribution and dropping a constant. \cite{Li:2006:PAD:1143844.1143917} and \cite{Mimno2007} gave collapsed Gibbs algorithms for four-layered PAM and hPAM models but not for arbitrary DAGs. This algorithm uses an arbitrary PAM-based hierarchical structure as described in Appendix \ref{Generalizations}. In Algorithm \ref{Collapsed PAM-Based Hierarchical Topic Model Algorithm}, ${n}_{\boldsymbol{i}k}^{(\ell,j)}$ is the count of customers assigned to topic $k$ in the mode $j$ at level $\ell$ and parent topics $\boldsymbol{i}$.

\begin{algorithm}[ht]
\small
\For{$\ell=1,\cdots,L$}{
\For{$j=1,\cdots,p$}{
\uIf{$\ell\neq 1$ or $j$ is not a root mode}{
\For{$x=1,\cdots,d_0$}{
$P\left(\boldsymbol{Y}_x^{(j)}|\boldsymbol{Y}_{-x}^{(j)},\boldsymbol{c}^{(j)},\boldsymbol{Z}^{(j)}\right)=\frac{\Gamma\left(m_{k,(\cdot)}^{(j),-x}+d_j\beta^{(j)}\right)}{\prod_y\Gamma\left(m_{k,y}^{(j),-x}+\beta^{(j)}\right)}\frac{\prod_y\Gamma\left(m_{k,y}^{(j)}+\beta^{(j)}\right)}{\Gamma\left(m_{k,(\cdot)}^{(j)}+d_j\beta^{(j)}\right)}$ \\
$P\left(c_{x,\ell}^{(j)}=k|\boldsymbol{Y}^{(j)},\boldsymbol{c}_{-x,\ell}^{(j)},\boldsymbol{Z}^{(j)}\right)\propto P\left(\boldsymbol{Y}_x^{(j)}|\boldsymbol{Y}_{-x}^{(j)},\boldsymbol{c}_\ell^{(j)},\boldsymbol{Z}^{(j)}\right) \left(\gamma_{\boldsymbol{i}k}^{(\ell,j)} + n_{\boldsymbol{i}k}^{(\ell,j)}\right)$
}
}
}
}
\caption{Collapsed PAM-Based Hierarchical Topic Model Algorithm}
\label{Collapsed PAM-Based Hierarchical Topic Model Algorithm}
\end{algorithm}

\section{Model Evaluation}
\label{Model Evaluation}

First, we describe the cancer and ASD data used to train and analyze our models in Section \ref{data sets}. We compare the coherence of various models trained on all three data sets in Section \ref{Coherence}. We analyze a sample set of topics and show that they correspond with known risk factors in Section \ref{Most Prevalent Pathways}. We compare the run times and algorithms' complexity in Section \ref{Complexity}. Then, we use a held-out non-parametric likelihood estimate to compare our cancer models to each other, described in Appendix \ref{Likelihood}. In Appendix \ref{Classsification}, we discuss classification results for cancer.

\subsection{Data Sets}
\label{data sets}

The first data set we examined is from The Cancer Genome Atlas and contains patients with one of four common types of cancer: breast, lung, prostate, or colorectal. Genetic factors are believed to affect the patient's risk of developing cancer. This data has a clear hierarchical structure: variants are on genes, which are part of biological pathways,\footnote{We define a pathway as a set of genes working together for a specific biological function. We do not consider the nature of interactions between genes, only their membership in the pathway.} which are within patients. Furthermore, pathways and genetic variants can be hierarchically grouped based on functions and interactions. An HBT decomposition would help classify types of cancer, where we have two modes of feature variables (genetic variants and pathways) and a clear hierarchical structure of groupings of genetic variants and pathways. We used the Reactome pathway data set \citep{pmid29145629,pmid24243840} to determine which genes are in each pathway.

Another data set we analyzed is from the National Database for Autism Research and contains paired siblings -- one with ASD and another without -- and counts of their genetic variants. We used the same process for including pathway information as the cancer data. ASD is a group of neurodevelopment disorders defined by a range of behavioral patterns and difficulty with social interaction. The Centers for Disease Control and Prevention estimates that 1 in 68 children have ASD, but it is more prevalent in boys than girls. Currently, ASD is typically diagnosed by parent and doctor observation of a child's behavior and development.
Experts believe ASD is caused by various genetic and environmental factors \citep{asdFact}.

More details about the data are in Appendix \ref{A. data sets}.

\subsection{Coherence}
\label{Coherence}

We compared the coherence of models trained on the cancer and ASD data. \cite{xing-etal-2019-evaluating} stated that co-occurrence and posterior-based methods are common evaluation techniques for LDA models. The main posterior method, topic stability, can be undermined by high-frequency genes and pathways (in our case), causing poor topics to have high stability. \cite{newman-etal-2010-automatic} showed that pointwise mutual information (PMI/UCI) consistently outperformed other methods in correlation with human subjects. However, we also provide \citepos{Mimno11} measure (UMass), a coherence measure intended for use without an external reference corpus. For both measures, the higher the value, the more coherent a topic or set of topics. Details on our implementation of these measures are in Appendix \ref{Coherence Measure}.

In practice, clinicians could examine the top genes or pathways in each topic to determine their structure. In Section \ref{Most Prevalent Pathways}, we analyze an example model in this fashion. However, this analysis is subjective, while coherence gives a similar objective comparison of models.

While most of the models we analyzed were Bayesian and hierarchical (using the independent trees model), the CP TensorLy model \citep{JMLR:v20:18-277} was deterministic and non-hierarchical.
In our tables, we bolded the best-performing measures for emphasis. We further explain our methodology in Appendix \ref{Coherence Models}.

\begin{figure}[t]
    \centering
    \includegraphics[scale=.94]{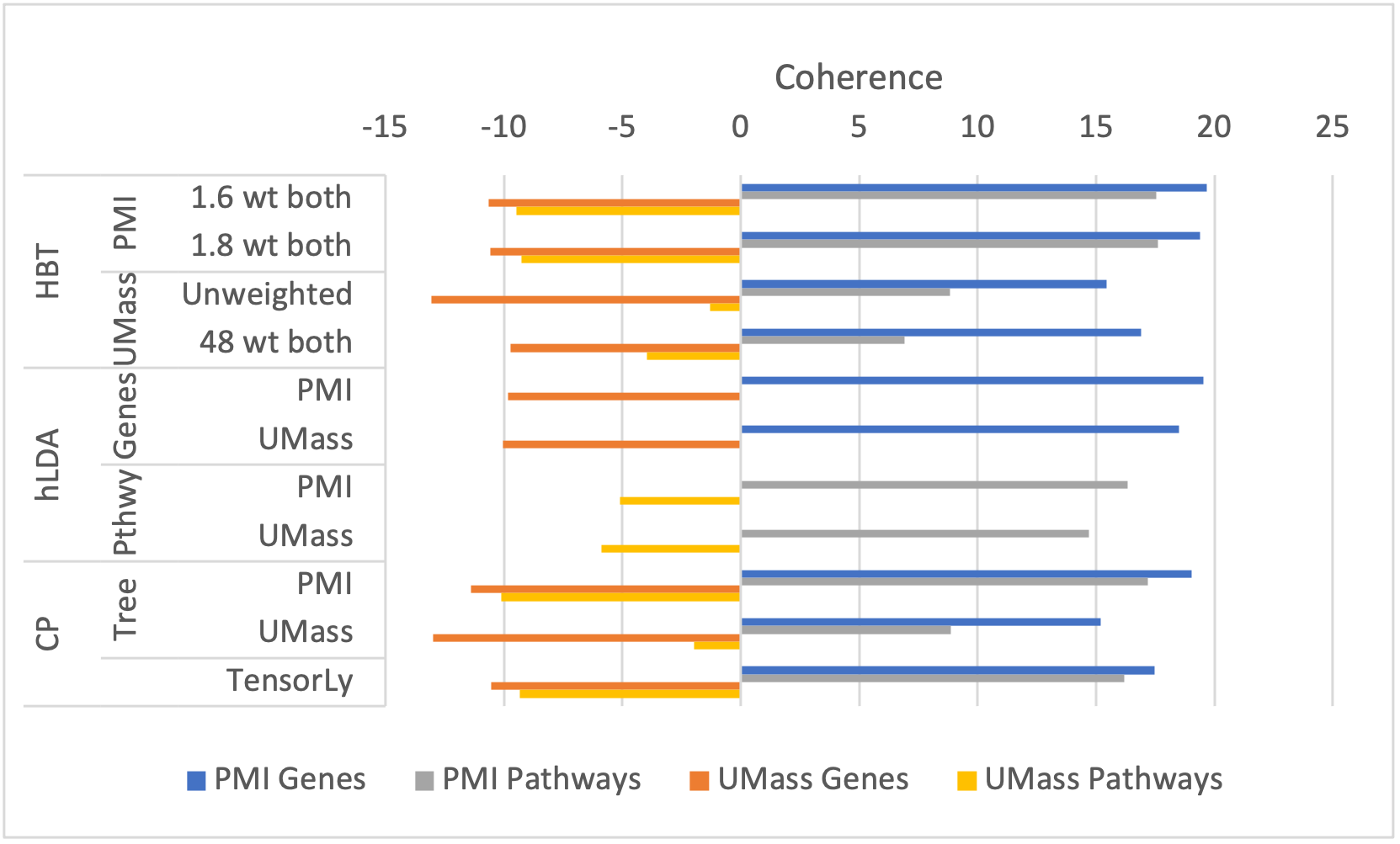}
    \caption{Independent Tree HBT models are more coherent than comparable models on cancer data. Our best HBT models outperformed the hLDA, CP tree, and CP TensorLy baselines on three-of-four coherence measures (using the mean over ten cross-validation folds).}
    \label{Cancer Coherence}
\end{figure}

\begin{figure}[t]
    \centering
    \includegraphics[scale=.94]{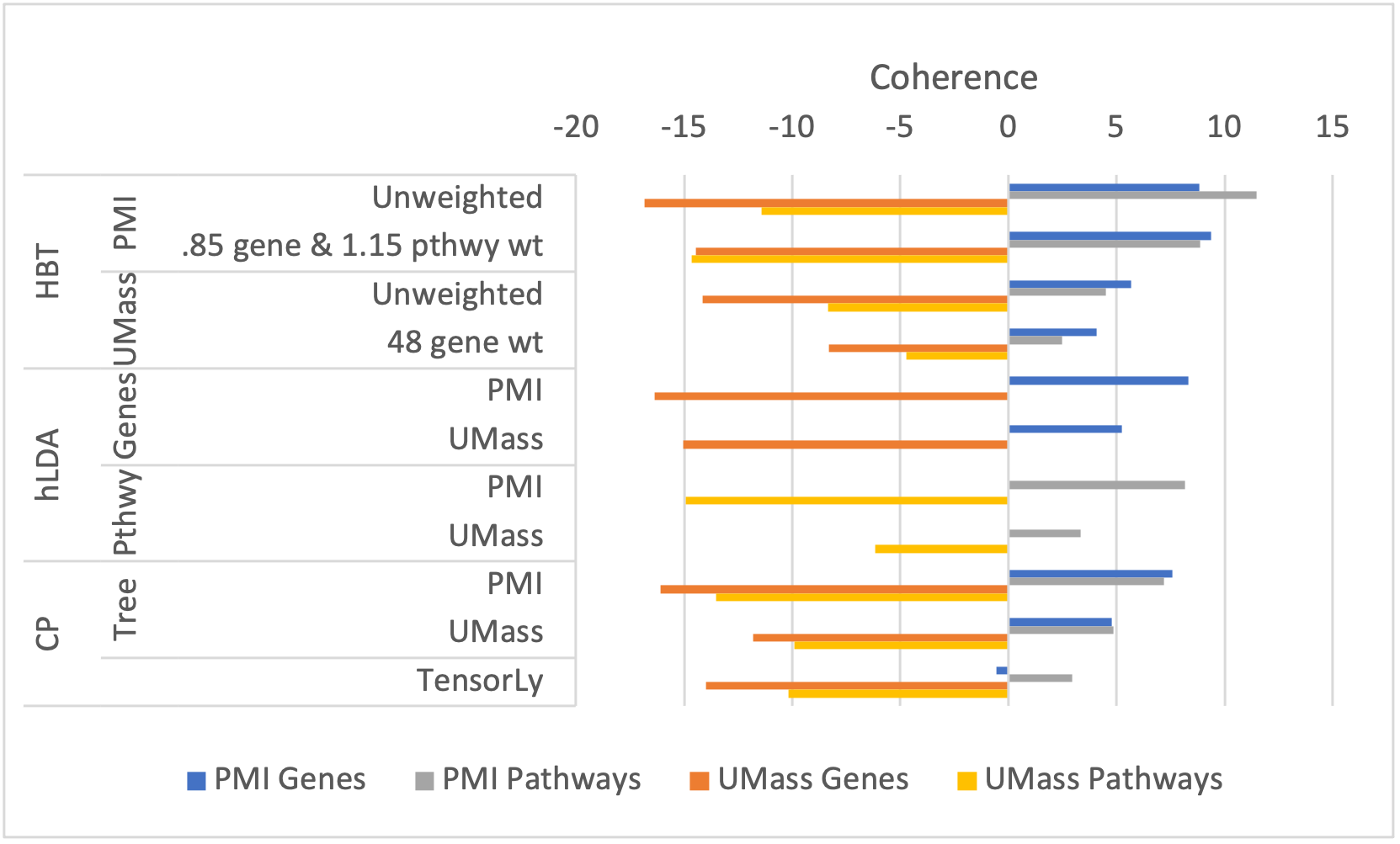}
    \caption{Independent Tree HBT models are more coherent than comparable models on ASD data. Our best HBT models outperformed the hLDA, CP tree, and CP TensorLy baselines on all four coherence measures (using the mean over ten cross-validation folds).}
    \label{ASD Coherence}
\end{figure}

Figures \ref{Cancer Coherence} and \ref{ASD Coherence} show that the best HBT models' topics were more coherent than the baseline models' (hLDA, CP tree, and CP TensorLy) on cancer and ASD data. On cancer data, our most significant improvement was on UMass pathways coherence. The unweighted HBT model using UMass coherence outperformed the best baseline, the CP tree model using UMass coherence, by 34.52\%. However, on the other three measures, the best HBT models outperformed the best baseline models by only 0.76\%-2.44\%. On ASD data, we saw more significant improvement in all four measures. Our smallest relative improvement was on PMI gene coherence, where the HBT model using PMI coherence and .85 gene and 1.15 pathway weighting outperformed hLDA using PMI coherence by 12.65\%. Our most significant improvement margin was on PMI pathway coherence, where the unweighted HBT model using PMI coherence outperformed hLDA using PMI coherence by 40.56\%. While no model performed best on all measures, the HBT models using PMI coherence outperformed the baseline models on all PMI coherence measures, and the weighted HBT UMass models outperformed the baselines on all UMass coherence measures.

In Appendix \ref{Reuters}, we applied an HBT structure to a different context, natural language processing, and achieved promising results. Here, our tensor consisted of articles, phrases, and words. While HBT performed worse than CP TensorLy on UMass phrase coherence, the HBT model using UMass coherence outperformed the best baseline models by up to 58\% on the other three coherence measures.

Generally, we found that HBT generates more coherent (and thus more interpretable) topics than other models. The choice of a specific model depends on the coherence measure, relative priority and structure of modes, number of topics, memory requirements, etc. Other models have their drawbacks. hLDA can only evaluate one mode at a time. CP tree imposes a more rigid structure, with a single topic hierarchy for both modes. CP Tensorly does not provide a hierarchical topic structure and is not as well suited for sparse counting tensors as Bayesian models.

\subsection{Most Prevalent Pathways}
\label{Most Prevalent Pathways}

\begin{table}[t]
\caption{Top Five Pathways in HBT ASD Model's Topics}
\label{top pathways}
\centering
\small
\begin{tabular}{Y{35mm}|Y{35mm}|Y{35mm}}
    Immune Function & Neurotransmission & Lipid Homeostasis  \\ \hline
    Interferon gamma signaling & Glutamate Neurotransmitter Release Cycle & Cell-extracellular matrix interactions \\
    Downstream TCR signaling & Intrinsic Pathway of Fibrin Clot Formation & CREB phosphorylation through the activation of Ras \\
    Phosphorylation of CD3 and TCR zeta chains & Ligand-dependent caspase activation & Post-translational protein phosphorylation \\
    MHC class II antigen presentation & Neurexins and neuroligins & BC transporters in lipid homeostasis \\
    Generation of second messenger molecules & Axonal growth inhibition (RHOA activation) & Defective Mismatch Repair Associated With MSH6 \\
\end{tabular}
\end{table}

Here, we examine three topics from a HBT ASD model. Their most prevalent pathways (shown in Table \ref{top pathways}) correspond to immune function, neurotransmission, and lipid homeostasis. The immune function \citep{Hughes2018, Meltzer2016, Onore2011} and neurotransmission \citep{Quaak2013, Cetin15} correspond to known ASD risk factors while growing literature suggests dyslipidemia may contribute to the development of ASD \citep{Luo2020}. The structure of these topics, with immune function as the root and neurotransmission and lipid homeostasis as branches, may provide additional insight into the relative importance of these factors, their prevalence, and their interrelationships. The six biomarker categories used by \cite{Abruzzo2015} were "neurotransmitters and neurotrophins, oxidative stress markers, fatty acids and phospholipids, inflammation markers, metabolites, toxic biomarkers, and metals and cations." Two of these categories overlap with the topics we found (neurotransmission and lipid homeostasis), supporting the assertion that HBT's topics are clinically relevant.

\subsection{Complexity and Run Time}
\label{Complexity}

\begin{figure*}[t]
    \centering
    \subfigure{
        \centering
        \includegraphics[scale=0.4]{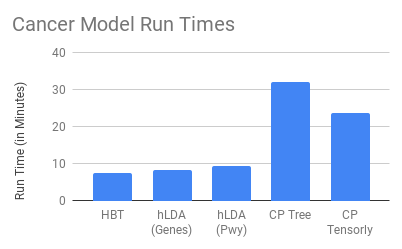}
    }
    \subfigure{
        \centering
        \includegraphics[scale=0.4]{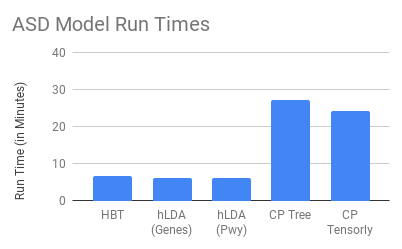}
    }
    \caption{Time to train a decomposition model on Cancer and ASD data.}
    \label{Run Times}
\end{figure*}

Figure \ref{Run Times} shows that the training time required for our HBT model is closer to those of hLDA than the CP models. Our model had the quickest training time on the cancer data and only lagged slightly behind the hLDA models on ASD. This is likely because the auxiliary matrices $\boldsymbol{\psi}$ and hierarchical trees are smaller as core tensor $\boldsymbol{\phi}$ condenses the information. In HBT, a 25x20 core tensor can convey 500 patient topics, 25 gene topics, and 20 pathway topics, while 500 patient topics in CP correspond to 500 gene topics and 500 pathway topics. This shows that our model scales better to multiple modes than comparable baselines. The run times may vary based on hardware and software (see \ref{Experiments}).

\section{Conclusion}
\label{Conclusion}

We developed methods for a Bayesian Tucker decomposition and designed strategies to incorporate dependent and independent hierarchical topic models. The independent trees model uses independent CRPs to generate hierarchical structures for each mode, while the PAM-based model creates a single hierarchical structure across modes. Furthermore, we presented a blueprint for generalizing these models to more than two feature modes by extending our independent and PAM-based hierarchical structures to account for more complex mode dependencies. We generalized the properties of CRP to multiple modes and proved that the strict versions of these properties cannot apply to these hierarchical models. Additionally, we derived a collapsed Gibbs sampler for the Bayesian Tucker decomposition with an arbitrary number of feature modes. Lastly, we trained our models on real-world examples and found that our models' topics were more coherent than existing methods.

In the context of genetic data, HBT describes each patient by their gene and pathway topics. These topics are arranged hierarchically so that the lower-level topics more finely differentiate groups of patients. Since genes make up pathways, PAM takes advantage of this relationship and provides a combined hierarchy of the gene and pathway topics. This model can analyze patient groups based on their shared genetic variants and pathways.

\section*{Acknowledgements}

Research reported in this publication was supported, in part, by the National Library of Medicine  [grant number T32LM012203]. The content is solely the responsibility of the authors and does not necessarily represent the official views of the National Institutes of Health.

\bibliographystyle{elsarticle-num-names}
\bibliography{bib}

\begin{thebibliography}{53}
\expandafter\ifx\csname natexlab\endcsname\relax\def\natexlab#1{#1}\fi
\providecommand{\url}[1]{\texttt{#1}}
\providecommand{\href}[2]{#2}
\providecommand{\path}[1]{#1}
\providecommand{\DOIprefix}{doi:}
\providecommand{\ArXivprefix}{arXiv:}
\providecommand{\URLprefix}{URL: }
\providecommand{\Pubmedprefix}{pmid:}
\providecommand{\doi}[1]{\href{http://dx.doi.org/#1}{\path{#1}}}
\providecommand{\Pubmed}[1]{\href{pmid:#1}{\path{#1}}}
\providecommand{\bibinfo}[2]{#2}
\ifx\xfnm\relax \def\xfnm[#1]{\unskip,\space#1}\fi
\bibitem[{Kolda and Bader(2009)}]{kolda2009tensor}
\bibinfo{author}{T.~G. Kolda}, \bibinfo{author}{B.~W. Bader},
\newblock \bibinfo{title}{Tensor decompositions and applications},
\newblock \bibinfo{journal}{SIAM review} \bibinfo{volume}{51} (\bibinfo{year}{2009}) \bibinfo{pages}{455--500}.
\bibitem[{Luo et~al.(2017{\natexlab{a}})Luo, Wang, and Szolovits}]{luo2017tensor}
\bibinfo{author}{Y.~Luo}, \bibinfo{author}{F.~Wang}, \bibinfo{author}{P.~Szolovits},
\newblock \bibinfo{title}{Tensor factorization toward precision medicine},
\newblock \bibinfo{journal}{Briefings in bioinformatics} \bibinfo{volume}{18} (\bibinfo{year}{2017}{\natexlab{a}}) \bibinfo{pages}{511--514}.
\bibitem[{Luo et~al.(2017{\natexlab{b}})Luo, Ahmad, and Shah}]{luo2017tensor2}
\bibinfo{author}{Y.~Luo}, \bibinfo{author}{F.~S. Ahmad}, \bibinfo{author}{S.~J. Shah},
\newblock \bibinfo{title}{Tensor factorization for precision medicine in heart failure with preserved ejection fraction},
\newblock \bibinfo{journal}{Journal of cardiovascular translational research} \bibinfo{volume}{10} (\bibinfo{year}{2017}{\natexlab{b}}) \bibinfo{pages}{305--312}.
\bibitem[{Dunson and Xing(2009)}]{dunson_xing_2009}
\bibinfo{author}{D.~B. Dunson}, \bibinfo{author}{C.~Xing},
\newblock \bibinfo{title}{Nonparametric {B}ayes modeling of multivariate categorical data},
\newblock \bibinfo{journal}{Journal of the American Statistical Association} \bibinfo{volume}{104} (\bibinfo{year}{2009}) \bibinfo{pages}{1042--1051}.
\bibitem[{Zhou et~al.(2015)Zhou, Bhattacharya, Herring, and Dunson}]{zhou_bhattacharya_herring_dunson_2015}
\bibinfo{author}{J.~Zhou}, \bibinfo{author}{A.~Bhattacharya}, \bibinfo{author}{A.~H. Herring}, \bibinfo{author}{D.~B. Dunson},
\newblock \bibinfo{title}{{B}ayesian factorizations of big sparse tensors},
\newblock \bibinfo{journal}{Journal of the American Statistical Association} \bibinfo{volume}{110} (\bibinfo{year}{2015}) \bibinfo{pages}{1562--1576}.
\bibitem[{Yang and Dunson(2016)}]{yang_dunson_2016}
\bibinfo{author}{Y.~Yang}, \bibinfo{author}{D.~B. Dunson},
\newblock \bibinfo{title}{{B}ayesian conditional tensor factorizations for high-dimensional classification},
\newblock \bibinfo{journal}{Journal of the American Statistical Association} \bibinfo{volume}{111} (\bibinfo{year}{2016}) \bibinfo{pages}{656--669}.
\bibitem[{Hoffman et~al.(2010)Hoffman, Bach, and Blei}]{NIPS2010_3902}
\bibinfo{author}{M.~Hoffman}, \bibinfo{author}{F.~R. Bach}, \bibinfo{author}{D.~M. Blei},
\newblock \bibinfo{title}{Online learning for latent {D}irichlet allocation},
\newblock in: \bibinfo{editor}{J.~D. Lafferty}, \bibinfo{editor}{C.~K.~I. Williams}, \bibinfo{editor}{J.~Shawe-Taylor}, \bibinfo{editor}{R.~S. Zemel}, \bibinfo{editor}{A.~Culotta} (Eds.), \bibinfo{booktitle}{Advances in Neural Information Processing Systems 23}, \bibinfo{publisher}{Curran Associates, Inc.}, \bibinfo{year}{2010}, pp. \bibinfo{pages}{856--864}.
\bibitem[{Buntine(2002)}]{Buntine2002}
\bibinfo{author}{W.~Buntine}, \bibinfo{title}{Variational Extensions to EM and Multinomial PCA}, \bibinfo{publisher}{Springer Berlin Heidelberg}, \bibinfo{address}{Berlin, Heidelberg}, \bibinfo{year}{2002}, pp. \bibinfo{pages}{23--34}.
\bibitem[{Schein et~al.(2016)Schein, Zhou, Blei, and Wallach}]{schein_zhou_blei_wallach_2016}
\bibinfo{author}{A.~Schein}, \bibinfo{author}{M.~Zhou}, \bibinfo{author}{D.~M. Blei}, \bibinfo{author}{H.~Wallach},
\newblock \bibinfo{title}{{B}ayesian {P}oisson tucker decomposition for learning the structure of international relations},
\newblock in: \bibinfo{booktitle}{Proceedings of the 33rd International Conference on Machine Learning}, ICML'16, \bibinfo{publisher}{JMLR.org}, \bibinfo{year}{2016}, pp. \bibinfo{pages}{2810--2819}.
\bibitem[{Xu et~al.(2012)Xu, Yan, and Qi}]{ICML2012Xu_543}
\bibinfo{author}{Z.~Xu}, \bibinfo{author}{F.~Yan}, \bibinfo{author}{A.~Qi},
\newblock \bibinfo{title}{Infinite tucker decomposition: Nonparametric {B}ayesian models for multiway data analysis},
\newblock in: \bibinfo{editor}{J.~Langford}, \bibinfo{editor}{J.~Pineau} (Eds.), \bibinfo{booktitle}{Proceedings of the 29th International Conference on Machine Learning}, ICML '07, \bibinfo{publisher}{ACM}, \bibinfo{address}{New York, NY, USA}, \bibinfo{year}{2012}, pp. \bibinfo{pages}{1023--1030}.
\bibitem[{Chi and Kolda(2012)}]{chi2012tensors}
\bibinfo{author}{E.~C. Chi}, \bibinfo{author}{T.~G. Kolda},
\newblock \bibinfo{title}{On tensors, sparsity, and nonnegative factorizations},
\newblock \bibinfo{journal}{SIAM Journal on Matrix Analysis and Applications} \bibinfo{volume}{33} (\bibinfo{year}{2012}) \bibinfo{pages}{1272--1299}.
\bibitem[{Hackbusch and K{\"u}hn(2009)}]{Hackbusch2009}
\bibinfo{author}{W.~Hackbusch}, \bibinfo{author}{S.~K{\"u}hn},
\newblock \bibinfo{title}{A new scheme for the tensor representation},
\newblock \bibinfo{journal}{Journal of Fourier Analysis and Applications} \bibinfo{volume}{15} (\bibinfo{year}{2009}) \bibinfo{pages}{706--722}.
\bibitem[{Grasedyck(2010)}]{grasedyck_2010}
\bibinfo{author}{L.~Grasedyck},
\newblock \bibinfo{title}{Hierarchical singular value decomposition of tensors},
\newblock \bibinfo{journal}{Society for Industrial and Applied Mathematics Journal on Matrix Analysis and Applications} \bibinfo{volume}{31} (\bibinfo{year}{2010}) \bibinfo{pages}{2029--2054}.
\bibitem[{Song et~al.(2013)Song, Ishteva, Parikh, Xing, and Park}]{pmlr-v28-song13}
\bibinfo{author}{L.~Song}, \bibinfo{author}{M.~Ishteva}, \bibinfo{author}{A.~Parikh}, \bibinfo{author}{E.~Xing}, \bibinfo{author}{H.~Park},
\newblock \bibinfo{title}{Hierarchical tensor decomposition of latent tree graphical models},
\newblock in: \bibinfo{editor}{S.~Dasgupta}, \bibinfo{editor}{D.~McAllester} (Eds.), \bibinfo{booktitle}{Proceedings of the 30th International Conference on Machine Learning}, volume~\bibinfo{volume}{28} of \textit{\bibinfo{series}{Proceedings of Machine Learning Research}}, \bibinfo{publisher}{PMLR}, \bibinfo{address}{Atlanta, Georgia, USA}, \bibinfo{year}{2013}, pp. \bibinfo{pages}{334--342}.
\bibitem[{Schifanella et~al.(2014)Schifanella, Candan, and Sapino}]{Schifanella:2014:MTD:2630935.2532169}
\bibinfo{author}{C.~Schifanella}, \bibinfo{author}{K.~S. Candan}, \bibinfo{author}{M.~L. Sapino},
\newblock \bibinfo{title}{Multiresolution tensor decompositions with mode hierarchies},
\newblock \bibinfo{journal}{Association for Computing Machinery Transactions on Knowledge Discovovery from Data} \bibinfo{volume}{8} (\bibinfo{year}{2014}) \bibinfo{pages}{10:1--10:38}.
\bibitem[{Teh et~al.(2006)Teh, Jordan, Beal, and Blei}]{teh2006hierarchical}
\bibinfo{author}{Y.~W. Teh}, \bibinfo{author}{M.~I. Jordan}, \bibinfo{author}{M.~J. Beal}, \bibinfo{author}{D.~M. Blei},
\newblock \bibinfo{title}{Hierarchical {D}irichlet processes},
\newblock \bibinfo{journal}{Journal of the american statistical association} \bibinfo{volume}{101} (\bibinfo{year}{2006}) \bibinfo{pages}{1566--1581}.
\bibitem[{Kong et~al.(2012)}]{kong2012}
\bibinfo{author}{S.~W. Kong}, et~al.,
\newblock \bibinfo{title}{Characteristics and predictive value of blood transcriptome signature in males with autism spectrum disorders},
\newblock \bibinfo{journal}{PLOS ONE} \bibinfo{volume}{7} (\bibinfo{year}{2012}) \bibinfo{pages}{1--13}.
\bibitem[{Duda et~al.(2018)}]{Duda2018}
\bibinfo{author}{M.~Duda}, et~al.,
\newblock \bibinfo{title}{Brain-specific functional relationship networks inform autism spectrum disorder gene prediction},
\newblock \bibinfo{journal}{Translational Psychiatry} \bibinfo{volume}{8} (\bibinfo{year}{2018}) \bibinfo{pages}{56}.
\bibitem[{Abruzzo et~al.(2015)Abruzzo, Ghezzo, Bolotta, Ferreri, Minguzzi, Vignini, Visconti, and Marini}]{Abruzzo2015}
\bibinfo{author}{P.~M. Abruzzo}, \bibinfo{author}{A.~Ghezzo}, \bibinfo{author}{A.~Bolotta}, \bibinfo{author}{C.~Ferreri}, \bibinfo{author}{R.~Minguzzi}, \bibinfo{author}{A.~Vignini}, \bibinfo{author}{P.~Visconti}, \bibinfo{author}{M.~Marini},
\newblock \bibinfo{title}{{Perspective Biological Markers for Autism Spectrum Disorders: Advantages of the Use of Receiver Operating Characteristic Curves in Evaluating Marker Sensitivity and Specificity}},
\newblock \bibinfo{journal}{Disease Markers} \bibinfo{volume}{2015} (\bibinfo{year}{2015}) \bibinfo{pages}{329607}.
\bibitem[{Goeman et~al.(2004)Goeman, van~de Geer, de~Kort, and van Houwelingen}]{bioinformatics/btg382}
\bibinfo{author}{J.~J. Goeman}, \bibinfo{author}{S.~A. van~de Geer}, \bibinfo{author}{F.~de~Kort}, \bibinfo{author}{H.~C. van Houwelingen},
\newblock \bibinfo{title}{{A global test for groups of genes: testing association with a clinical outcome}},
\newblock \bibinfo{journal}{Bioinformatics} \bibinfo{volume}{20} (\bibinfo{year}{2004}) \bibinfo{pages}{93--99}. \DOIprefix\doi{10.1093/bioinformatics/btg382}.
\bibitem[{Subramanian et~al.(2005)Subramanian, Tamayo, Mootha, Mukherjee, Ebert, Gillette, Paulovich, Pomeroy, Golub, Lander, and Mesirov}]{gsea}
\bibinfo{author}{A.~Subramanian}, \bibinfo{author}{P.~Tamayo}, \bibinfo{author}{V.~K. Mootha}, \bibinfo{author}{S.~Mukherjee}, \bibinfo{author}{B.~L. Ebert}, \bibinfo{author}{M.~A. Gillette}, \bibinfo{author}{A.~Paulovich}, \bibinfo{author}{S.~L. Pomeroy}, \bibinfo{author}{T.~R. Golub}, \bibinfo{author}{E.~S. Lander}, \bibinfo{author}{J.~P. Mesirov},
\newblock \bibinfo{title}{Gene set enrichment analysis: A knowledge-based approach for interpreting genome-wide expression profiles},
\newblock \bibinfo{journal}{Proceedings of the National Academy of Sciences} \bibinfo{volume}{102} (\bibinfo{year}{2005}) \bibinfo{pages}{15545--15550}. \DOIprefix\doi{10.1073/pnas.0506580102}.
\bibitem[{Gill and Marchini(2020)}]{Gill2020.11.30.403907}
\bibinfo{author}{C.~C. Gill}, \bibinfo{author}{J.~Marchini},
\newblock \bibinfo{title}{Four-dimensional sparse bayesian tensor decomposition for gene expression data},
\newblock \bibinfo{journal}{bioRxiv}  (\bibinfo{year}{2020}). \URLprefix \url{https://www.biorxiv.org/content/early/2020/11/30/2020.11.30.403907}. \DOIprefix\doi{10.1101/2020.11.30.403907}.
\bibitem[{Liu et~al.(2022)Liu, Cheng, Jin, and Hu}]{LIU2022103958}
\bibinfo{author}{Q.~Liu}, \bibinfo{author}{B.~Cheng}, \bibinfo{author}{Y.~Jin}, \bibinfo{author}{P.~Hu},
\newblock \bibinfo{title}{Bayesian tensor factorization-drive breast cancer subtyping by integrating multi-omics data},
\newblock \bibinfo{journal}{Journal of Biomedical Informatics} \bibinfo{volume}{125} (\bibinfo{year}{2022}) \bibinfo{pages}{103958}. \URLprefix \url{https://www.sciencedirect.com/science/article/pii/S1532046421002872}. \DOIprefix\doi{10.1016/j.jbi.2021.103958}.
\bibitem[{Liu et~al.(2023)Liu, Chakraborty, Qin, Kundu, and {The Alzheimer's Disease Neuroimaging Initiative}}]{fnins.2023.1212218}
\bibinfo{author}{Y.~Liu}, \bibinfo{author}{N.~Chakraborty}, \bibinfo{author}{Z.~S. Qin}, \bibinfo{author}{S.~Kundu}, \bibinfo{author}{{The Alzheimer's Disease Neuroimaging Initiative}},
\newblock \bibinfo{title}{Integrative bayesian tensor regression for imaging genetics applications},
\newblock \bibinfo{journal}{Frontiers in Neuroscience} \bibinfo{volume}{17} (\bibinfo{year}{2023}). \URLprefix \url{https://www.frontiersin.org/journals/neuroscience/articles/10.3389/fnins.2023.1212218}. \DOIprefix\doi{10.3389/fnins.2023.1212218}.
\bibitem[{Ma and Ma(2024)}]{journal.pcbi.1012287}
\bibinfo{author}{Y.~Ma}, \bibinfo{author}{Y.~Ma},
\newblock \bibinfo{title}{Kernel bayesian logistic tensor decomposition with automatic rank determination for predicting multiple types of mirna-disease associations},
\newblock \bibinfo{journal}{PLOS Computational Biology} \bibinfo{volume}{20} (\bibinfo{year}{2024}) \bibinfo{pages}{1--23}. \DOIprefix\doi{10.1371/journal.pcbi.1012287}.
\bibitem[{Blei et~al.(2003)Blei, Jordan, Griffiths, and Tenenbaum}]{NIPS2003_2466}
\bibinfo{author}{D.~M. Blei}, \bibinfo{author}{M.~I. Jordan}, \bibinfo{author}{T.~L. Griffiths}, \bibinfo{author}{J.~B. Tenenbaum},
\newblock \bibinfo{title}{Hierarchical topic models and the nested {C}hinese restaurant process},
\newblock in: \bibinfo{booktitle}{Proceedings of the 16th International Conference on Neural Information Processing Systems}, NIPS'03, \bibinfo{publisher}{MIT Press}, \bibinfo{address}{Cambridge, MA, USA}, \bibinfo{year}{2003}, p. \bibinfo{pages}{17–24}.
\bibitem[{Li and McCallum(2006)}]{Li:2006:PAD:1143844.1143917}
\bibinfo{author}{W.~Li}, \bibinfo{author}{A.~McCallum},
\newblock \bibinfo{title}{Pachinko allocation: {DAG}-structured mixture models of topic correlations},
\newblock in: \bibinfo{booktitle}{Proceedings of the 23rd International Conference on Machine Learning}, ICML '06, \bibinfo{publisher}{ACM}, \bibinfo{address}{New York, NY, USA}, \bibinfo{year}{2006}, pp. \bibinfo{pages}{577--584}.
\bibitem[{Mimno et~al.(2007)Mimno, Li, and McCallum}]{Mimno2007}
\bibinfo{author}{D.~Mimno}, \bibinfo{author}{W.~Li}, \bibinfo{author}{A.~McCallum},
\newblock \bibinfo{title}{Mixtures of hierarchical topics with pachinko allocation},
\newblock in: \bibinfo{booktitle}{Proceedings of the 24th International Conference on Machine Learning}, ICML '07, \bibinfo{publisher}{Association for Computing Machinery}, \bibinfo{address}{New York, NY, USA}, \bibinfo{year}{2007}, p. \bibinfo{pages}{633–640}. \DOIprefix\doi{10.1145/1273496.1273576}.
\bibitem[{Blei(2007)}]{Blei2007}
\bibinfo{author}{D.~Blei}, \bibinfo{title}{{COS} 597c: {B}ayesian nonparametrics}, \bibinfo{year}{2007}. \URLprefix \url{https://www.cs.princeton.edu/courses/archive/fall07/cos597C/scribe/20070921.pdf}.
\bibitem[{Teh(2010)}]{Teh2010a}
\bibinfo{author}{Y.~W. Teh},
\newblock \bibinfo{title}{{D}irichlet processes},
\newblock in: \bibinfo{booktitle}{Encyclopedia of Machine Learning}, \bibinfo{publisher}{Springer}, \bibinfo{year}{2010}.
\bibitem[{Fabregat et~al.(2018)Fabregat, Jupe, Matthews, Sidiropoulos, Gillespie, Garapati, Haw, Jassal, Korninger, May, Milacic, Roca, Rothfels, Sevilla, Shamovsky, Shorser, Varusai, Viteri, Weiser, Wu, Stein, Hermjakob, and D'Eustachio}]{pmid29145629}
\bibinfo{author}{A.~Fabregat}, \bibinfo{author}{S.~Jupe}, \bibinfo{author}{L.~Matthews}, \bibinfo{author}{K.~Sidiropoulos}, \bibinfo{author}{M.~Gillespie}, \bibinfo{author}{P.~Garapati}, \bibinfo{author}{R.~Haw}, \bibinfo{author}{B.~Jassal}, \bibinfo{author}{F.~Korninger}, \bibinfo{author}{B.~May}, \bibinfo{author}{M.~Milacic}, \bibinfo{author}{C.~D. Roca}, \bibinfo{author}{K.~Rothfels}, \bibinfo{author}{C.~Sevilla}, \bibinfo{author}{V.~Shamovsky}, \bibinfo{author}{S.~Shorser}, \bibinfo{author}{T.~Varusai}, \bibinfo{author}{G.~Viteri}, \bibinfo{author}{J.~Weiser}, \bibinfo{author}{G.~Wu}, \bibinfo{author}{L.~Stein}, \bibinfo{author}{H.~Hermjakob}, \bibinfo{author}{P.~D'Eustachio},
\newblock \bibinfo{title}{The {R}eactome pathway knowledgebase},
\newblock \bibinfo{journal}{Nucleic Acids Research} \bibinfo{volume}{46} (\bibinfo{year}{2018}) \bibinfo{pages}{D649--D655}.
\bibitem[{Croft et~al.(2014)Croft, Mundo, Haw, Milacic, Weiser, Wu, Caudy, Garapati, Gillespie, Kamdar, Jassal, Jupe, Matthews, May, Palatnik, Rothfels, Shamovsky, Song, Williams, Birney, Hermjakob, Stein, and D'Eustachio}]{pmid24243840}
\bibinfo{author}{D.~Croft}, \bibinfo{author}{A.~F. Mundo}, \bibinfo{author}{R.~Haw}, \bibinfo{author}{M.~Milacic}, \bibinfo{author}{J.~Weiser}, \bibinfo{author}{G.~Wu}, \bibinfo{author}{M.~Caudy}, \bibinfo{author}{P.~Garapati}, \bibinfo{author}{M.~Gillespie}, \bibinfo{author}{M.~R. Kamdar}, \bibinfo{author}{B.~Jassal}, \bibinfo{author}{S.~Jupe}, \bibinfo{author}{L.~Matthews}, \bibinfo{author}{B.~May}, \bibinfo{author}{S.~Palatnik}, \bibinfo{author}{K.~Rothfels}, \bibinfo{author}{V.~Shamovsky}, \bibinfo{author}{H.~Song}, \bibinfo{author}{M.~Williams}, \bibinfo{author}{E.~Birney}, \bibinfo{author}{H.~Hermjakob}, \bibinfo{author}{L.~Stein}, \bibinfo{author}{P.~D'Eustachio},
\newblock \bibinfo{title}{The {R}eactome pathway knowledgebase},
\newblock \bibinfo{journal}{Nucleic Acids Research} \bibinfo{volume}{42} (\bibinfo{year}{2014}) \bibinfo{pages}{D472--477}.
\bibitem[{{National Institute of Neurological Disorders and Stroke}(2017)}]{asdFact}
\bibinfo{author}{{National Institute of Neurological Disorders and Stroke}}, \bibinfo{title}{Autism spectrum disorder fact sheet}, \bibinfo{year}{2017}.
\bibitem[{Xing et~al.(2019)Xing, Paul, and Carenini}]{xing-etal-2019-evaluating}
\bibinfo{author}{L.~Xing}, \bibinfo{author}{M.~J. Paul}, \bibinfo{author}{G.~Carenini},
\newblock \bibinfo{title}{Evaluating topic quality with posterior variability},
\newblock in: \bibinfo{editor}{K.~Inui}, \bibinfo{editor}{J.~Jiang}, \bibinfo{editor}{V.~Ng}, \bibinfo{editor}{X.~Wan} (Eds.), \bibinfo{booktitle}{Proceedings of the 2019 Conference on Empirical Methods in Natural Language Processing and the 9th International Joint Conference on Natural Language Processing (EMNLP-IJCNLP)}, \bibinfo{publisher}{Association for Computational Linguistics}, \bibinfo{address}{Hong Kong, China}, \bibinfo{year}{2019}, pp. \bibinfo{pages}{3471--3477}. \URLprefix \url{https://aclanthology.org/D19-1349}. \DOIprefix\doi{10.18653/v1/D19-1349}.
\bibitem[{Newman et~al.(2010)Newman, Lau, Grieser, and Baldwin}]{newman-etal-2010-automatic}
\bibinfo{author}{D.~Newman}, \bibinfo{author}{J.~H. Lau}, \bibinfo{author}{K.~Grieser}, \bibinfo{author}{T.~Baldwin},
\newblock \bibinfo{title}{Automatic evaluation of topic coherence},
\newblock in: \bibinfo{booktitle}{Human Language Technologies: The 2010 Annual Conference of the North {A}merican Chapter of the Association for Computational Linguistics}, \bibinfo{publisher}{Association for Computational Linguistics}, \bibinfo{address}{Los Angeles, California}, \bibinfo{year}{2010}, pp. \bibinfo{pages}{100--108}. \URLprefix \url{https://aclanthology.org/N10-1012}.
\bibitem[{Mimno et~al.(2011)Mimno, Wallach, Talley, Leenders, and McCallum}]{Mimno11}
\bibinfo{author}{D.~Mimno}, \bibinfo{author}{H.~M. Wallach}, \bibinfo{author}{E.~Talley}, \bibinfo{author}{M.~Leenders}, \bibinfo{author}{A.~McCallum},
\newblock \bibinfo{title}{Optimizing semantic coherence in topic models},
\newblock in: \bibinfo{booktitle}{Proceedings of the Conference on Empirical Methods in Natural Language Processing}, EMNLP '11, \bibinfo{publisher}{Association for Computational Linguistics}, \bibinfo{year}{2011}, p. \bibinfo{pages}{262–272}.
\bibitem[{Kossaifi et~al.(2019)Kossaifi, Panagakis, Anandkumar, and Pantic}]{JMLR:v20:18-277}
\bibinfo{author}{J.~Kossaifi}, \bibinfo{author}{Y.~Panagakis}, \bibinfo{author}{A.~Anandkumar}, \bibinfo{author}{M.~Pantic},
\newblock \bibinfo{title}{Tensor{L}y: Tensor learning in python},
\newblock \bibinfo{journal}{Journal of Machine Learning Research} \bibinfo{volume}{20} (\bibinfo{year}{2019}) \bibinfo{pages}{1--6}. \URLprefix \url{http://jmlr.org/papers/v20/18-277.html}.
\bibitem[{Hughes et~al.(2018)Hughes, Mills~Ko, Rose, and Ashwood}]{Hughes2018}
\bibinfo{author}{H.~K. Hughes}, \bibinfo{author}{E.~Mills~Ko}, \bibinfo{author}{D.~Rose}, \bibinfo{author}{P.~Ashwood},
\newblock \bibinfo{title}{Immune dysfunction and autoimmunity as pathological mechanisms in autism spectrum disorders},
\newblock \bibinfo{journal}{Frontiers in Cellular Neuroscience} \bibinfo{volume}{12} (\bibinfo{year}{2018}). \URLprefix \url{https://www.frontiersin.org/article/10.3389/fncel.2018.00405}. \DOIprefix\doi{10.3389/fncel.2018.00405}.
\bibitem[{Meltzer and Van~de Water(2016)}]{Meltzer2016}
\bibinfo{author}{A.~Meltzer}, \bibinfo{author}{J.~Van~de Water},
\newblock \bibinfo{title}{The role of the immune system in autism spectrum disorder},
\newblock \bibinfo{journal}{Neuropsychopharmacology} \bibinfo{volume}{42} (\bibinfo{year}{2016}). \URLprefix \url{https://pubmed.ncbi.nlm.nih.gov/27534269/}. \DOIprefix\doi{10.1038/npp.2016.158}.
\bibitem[{Onore et~al.(2011)Onore, Careaga, and Ashwood}]{Onore2011}
\bibinfo{author}{C.~Onore}, \bibinfo{author}{M.~Careaga}, \bibinfo{author}{P.~Ashwood},
\newblock \bibinfo{title}{The role of immune dysfunction in the pathophysiology of autism},
\newblock \bibinfo{journal}{Brain, behavior, and immunity} \bibinfo{volume}{26} (\bibinfo{year}{2011}). \URLprefix \url{https://pubmed.ncbi.nlm.nih.gov/21906670/}. \DOIprefix\doi{10.1016/j.bbi.2011.08.007}.
\bibitem[{Quaak et~al.(2013)Quaak, Brouns, and Van~de Bor}]{Quaak2013}
\bibinfo{author}{I.~Quaak}, \bibinfo{author}{M.~R. Brouns}, \bibinfo{author}{M.~Van~de Bor},
\newblock \bibinfo{title}{The dynamics of autism spectrum disorders: How neurotoxic compounds and neurotransmitters interact},
\newblock \bibinfo{journal}{International Journal of Environmental Research and Public Health} \bibinfo{volume}{10} (\bibinfo{year}{2013}). \URLprefix \url{https://www.ncbi.nlm.nih.gov/pmc/articles/PMC3774444/}. \DOIprefix\doi{10.3390/ijerph10083384}.
\bibitem[{Cetin et~al.(2015)Cetin, Tunca, Guney, and Iseri}]{Cetin15}
\bibinfo{author}{F.~H. Cetin}, \bibinfo{author}{H.~Tunca}, \bibinfo{author}{E.~Guney}, \bibinfo{author}{E.~Iseri},
\newblock \bibinfo{title}{Neurotransmitter systems in autism spectrum disorder},
\newblock in: \bibinfo{editor}{M.~Fitzgerald} (Ed.), \bibinfo{booktitle}{Autism Spectrum Disorder}, \bibinfo{publisher}{IntechOpen}, \bibinfo{address}{Rijeka}, \bibinfo{year}{2015}. \DOIprefix\doi{10.5772/59122}.
\bibitem[{Luo et~al.(2020)Luo, Eran, Palmer, Avillach, Levy-Moonshine, Szolovits, and Kohane}]{Luo2020}
\bibinfo{author}{Y.~Luo}, \bibinfo{author}{A.~Eran}, \bibinfo{author}{N.~Palmer}, \bibinfo{author}{P.~Avillach}, \bibinfo{author}{A.~Levy-Moonshine}, \bibinfo{author}{P.~Szolovits}, \bibinfo{author}{I.~Kohane},
\newblock \bibinfo{title}{A multidimensional precision medicine approach identifies an autism subtype characterized by dyslipidemia},
\newblock \bibinfo{journal}{Nature Medicine} \bibinfo{volume}{26} (\bibinfo{year}{2020}) \bibinfo{pages}{1375--1379}. \DOIprefix\doi{10.1038/s41591-020-1007-0}.
\bibitem[{Bader and Kolda(2021)}]{bader_kolda_2021}
\bibinfo{author}{B.~W. Bader}, \bibinfo{author}{T.~G. Kolda}, \bibinfo{title}{Tensor toolbox for {MATLAB}, version 3.2.1}, \bibinfo{year}{2021}. \URLprefix \url{http://www.tensortoolbox.org/}.
\bibitem[{Rehurek and Sojka(2011)}]{rehurek2011gensim}
\bibinfo{author}{R.~Rehurek}, \bibinfo{author}{P.~Sojka},
\newblock \bibinfo{title}{Gensim--python framework for vector space modelling},
\newblock \bibinfo{journal}{NLP Centre, Faculty of Informatics, Masaryk University, Brno, Czech Republic} \bibinfo{volume}{3} (\bibinfo{year}{2011}).
\bibitem[{Bird et~al.(2009)Bird, Klein, and Loper}]{bird2009natural}
\bibinfo{author}{S.~Bird}, \bibinfo{author}{E.~Klein}, \bibinfo{author}{E.~Loper}, \bibinfo{title}{Natural language processing with Python: analyzing text with the natural language toolkit}, \bibinfo{publisher}{"O'Reilly Media, Inc."}, \bibinfo{year}{2009}.
\bibitem[{Kossaifi et~al.(2019)Kossaifi, Panagakis, Anandkumar, and Pantic}]{tensorly}
\bibinfo{author}{J.~Kossaifi}, \bibinfo{author}{Y.~Panagakis}, \bibinfo{author}{A.~Anandkumar}, \bibinfo{author}{M.~Pantic},
\newblock \bibinfo{title}{Tensorly: Tensor learning in python},
\newblock \bibinfo{journal}{Journal of Machine Learning Research (JMLR)} \bibinfo{volume}{20} (\bibinfo{year}{2019}).
\bibitem[{Ding et~al.(2020)Ding, Wang, Li, Li, and Liu}]{ding2020more}
\bibinfo{author}{K.~Ding}, \bibinfo{author}{J.~Wang}, \bibinfo{author}{J.~Li}, \bibinfo{author}{D.~Li}, \bibinfo{author}{H.~Liu},
\newblock \bibinfo{title}{Be more with less: Hypergraph attention networks for inductive text classification},
\newblock in: \bibinfo{booktitle}{Proceedings of the 2020 Conference on Empirical Methods in Natural Language Processing (EMNLP)}, \bibinfo{year}{2020}, pp. \bibinfo{pages}{4927--4936}.
\bibitem[{Dua and Graff(2017)}]{Dua:2019}
\bibinfo{author}{D.~Dua}, \bibinfo{author}{C.~Graff}, \bibinfo{title}{{UCI} machine learning repository}, \bibinfo{year}{2017}. \URLprefix \url{http://archive.ics.uci.edu/ml}.
\bibitem[{Honnibal et~al.(2020)Honnibal, Montani, Van~Landeghem, and Boyd}]{spacy2}
\bibinfo{author}{M.~Honnibal}, \bibinfo{author}{I.~Montani}, \bibinfo{author}{S.~Van~Landeghem}, \bibinfo{author}{A.~Boyd},
\newblock \bibinfo{title}{spa{C}y: Industrial-strength natural language processing in python}  (\bibinfo{year}{2020}). \DOIprefix\doi{10.5281/zenodo.1212303}.
\bibitem[{Zhao et~al.(2015)Zhao, Zhang, and Cichocki}]{ZhaoZC14}
\bibinfo{author}{Q.~Zhao}, \bibinfo{author}{L.~Zhang}, \bibinfo{author}{A.~Cichocki},
\newblock \bibinfo{title}{Bayesian {CP} factorization of incomplete tensors with automatic rank determination},
\newblock \bibinfo{journal}{IEEE Transactions on Pattern Analysis and Machine Intelligence} \bibinfo{volume}{37} (\bibinfo{year}{2015}) \bibinfo{pages}{1751--1763}. \DOIprefix\doi{10.1109/TPAMI.2015.2392756}.
\bibitem[{Wang et~al.(2015)Wang, Chen, Ghosh, Denny, Kho, Chen, Malin, and Sun}]{Wang2015}
\bibinfo{author}{Y.~Wang}, \bibinfo{author}{R.~Chen}, \bibinfo{author}{J.~Ghosh}, \bibinfo{author}{J.~C. Denny}, \bibinfo{author}{A.~Kho}, \bibinfo{author}{Y.~Chen}, \bibinfo{author}{B.~A. Malin}, \bibinfo{author}{J.~Sun},
\newblock \bibinfo{title}{Rubik: Knowledge guided tensor factorization and completion for health data analytics},
\newblock \bibinfo{journal}{KDD : proceedings. International Conference on Knowledge Discovery {\&} Data Mining} \bibinfo{volume}{2015} (\bibinfo{year}{2015}) \bibinfo{pages}{1265--1274}. \URLprefix \url{https://pubmed.ncbi.nlm.nih.gov/31452969}. \DOIprefix\doi{10.1145/2783258.2783395}.
\bibitem[{Diggle and Gratton(1984)}]{Diggle}
\bibinfo{author}{P.~J. Diggle}, \bibinfo{author}{R.~J. Gratton},
\newblock \bibinfo{title}{{M}onte {C}arlo methods of inference for implicit statistical models},
\newblock \bibinfo{journal}{Journal of the Royal Statistical Society, Series B: Methodological} \bibinfo{volume}{46} (\bibinfo{year}{1984}) \bibinfo{pages}{193--227}.

\end{thebibliography}

\appendix
\appendixpage

\section{Non-Collapsed Algorithms}
\label{Non-Collapsed Algorithms}

\begin{algorithm}[H]
\small
\For{$x=1,\cdots,d_0$}{
Draw core tensor $\tilde{\boldsymbol\phi}_x\sim\text{Dir}(\boldsymbol\alpha+\boldsymbol{n}_x)\in \S_K$ \\
\For{$\boldsymbol k\in\K$}{
$\phi_{x\boldsymbol k}=\tilde{\phi}_{x\text{vec}(\boldsymbol k)}$ \\
}
}
\For{$j=1,\cdots,p$}{
\For{$k=1,\cdots,K_j$}{
Draw auxiliary matrices $\boldsymbol\psi_{k}^{(j)}\sim\text{Dir}\left(\boldsymbol\beta^{(j)}+\boldsymbol{m}^{(j)}_{h}\right)\in \S_{d_j}$ \\
}
}
\For{$x=1,\cdots,d_0$}{
\For{$i=1,\cdots,\lambda_x$}{
Compute topic probabilities $P\left(\boldsymbol z_{i}^{(x)}=\boldsymbol{k}|-\right)\propto\phi_{x\boldsymbol k}\psi_{ky_i}^{(j)}$ \\
Draw $\boldsymbol z_{i}^{(x)}$ from $P\left(\boldsymbol z_{i}^{(x)}=\boldsymbol{k}|-\right)$ \\
}
}
\caption{Non-Collapsed Bayesian Tucker Decomposition Gibbs Sampler}
\label{Non-Collapsed Bayesian Tucker Decomposition Gibbs Sampler}
\end{algorithm}

\begin{algorithm}[H]
\small
\For{$\ell=1,\cdots,L$}{
\For{$j=1,\cdots,p$}{
\uIf{$\ell\neq 1$ or $j$ is not a root mode}{
\For{$\boldsymbol{i}\in\{\text{possible parent topics}\}$}{
$\boldsymbol{P}_{\boldsymbol{i}}^{(\ell,j)}\sim\text{Dir}\left(\boldsymbol{\gamma}_{\boldsymbol{i}}^{(\ell,j)} + \boldsymbol{n}_{\boldsymbol{i}}^{(\ell,j)}\right)$ \\
}
\For{$x=1,\cdots,d_0$}{
$P\left(\boldsymbol{Y}_x^{(j)}|\boldsymbol{Y}_{-x}^{(j)},\boldsymbol{c}^{(j)},\boldsymbol{Z}^{(j)}\right)=\frac{\Gamma\left(m_{k,(\cdot)}^{(j),-x}+d_j\beta^{(j)}\right)}{\prod_y\Gamma\left(m_{k,y}^{(j),-x}+\beta^{(j)}\right)}\frac{\prod_y\Gamma\left(m_{k,y}^{(j)}+\beta^{(j)}\right)}{\Gamma\left(m_{k,(\cdot)}^{(j)}+d_j\beta^{(j)}\right)}$ \\
$P\left(c_{x,\ell}^{(j)}=k|\boldsymbol{Y}^{(j)},\boldsymbol{c}_{-x,\ell}^{(j)},\boldsymbol{Z}^{(j)}\right)\propto P\left(\boldsymbol{Y}_x^{(j)}|\boldsymbol{Y}_{-x}^{(j)},\boldsymbol{c}_\ell^{(j)},\boldsymbol{Z}^{(j)}\right) P_{\boldsymbol{i}k}^{(\ell,j)}$
}
}
}
}
\caption{Non-Collapsed PAM-Based Hierarchical Topic Model Algorithm}
\label{Non-Collapsed PAM-Based Hierarchical Topic Model Algorithm}
\end{algorithm}

\section{Collapsed Gibbs Derivation}
\label{Collapsed Gibbs}

Here, we derive equations for collapsed Gibbs sampling of a conditional Bayesian Tucker decomposition. This derivation is similar to that of LDA. We begin with the total probability of our model and integrate out $\boldsymbol\phi$ and $\boldsymbol\psi$:
\begin{equation*}
\begin{aligned}P(\boldsymbol Y, \boldsymbol Z|\boldsymbol\alpha, \boldsymbol\beta)= &\mathop{\mathlarger{\int}}_{\boldsymbol\phi}\mathop{\mathlarger{\int}}_{\boldsymbol\psi}P(\boldsymbol Y, \boldsymbol Z, \boldsymbol\phi, \boldsymbol\psi|\boldsymbol\alpha, \boldsymbol\beta)\ d\boldsymbol\psi\  d\boldsymbol\phi \\
= &\mathop{\mathlarger{\int}}_{\boldsymbol\psi}\prod\limits_{j=1}^p \prod\limits_{h=1}^{K_j}P\left(\boldsymbol\psi_{h}^{(j)}|\boldsymbol\beta\right) \prod\limits_{i=1}^{\lambda_{x}} P\left(\boldsymbol y_{xi}|\boldsymbol\psi_{\boldsymbol z_{xi}}\right) d\boldsymbol\psi \\
& \times \mathop{\mathlarger{\int}}_{\boldsymbol\phi} \prod\limits_{x=1}^{d_0} P\left(\boldsymbol\phi_{x}|\boldsymbol\alpha\right) \prod\limits_{i=1}^{\lambda_{x}} P\left(\boldsymbol z_{xi}|\boldsymbol\phi_{x}\right)d\boldsymbol\phi.\end{aligned}\end{equation*}
All $\boldsymbol\psi$'s and $\boldsymbol\phi$'s are independent of each other and thus can be treated separately. We first examine the $\boldsymbol\phi$'s:
\begin{equation*}\mathop{\mathlarger{\int}}_{\boldsymbol\phi} \prod\limits_{x=1}^{d_0} P\left(\boldsymbol\phi_{x}|\boldsymbol\alpha\right) \prod\limits_{i=1}^{\lambda_{x}} P\left(\boldsymbol z_{xi}|\boldsymbol\phi_{x}\right)d\boldsymbol\phi=\prod\limits_{x=1}^{d_0} \mathop{\mathlarger{\int}}_{\boldsymbol\phi_x} P\left(\boldsymbol\phi_{x}|\boldsymbol\alpha\right) \prod\limits_{i=1}^{\lambda_{x}} P\left(\boldsymbol z_{xi}|\boldsymbol\phi_{x}\right)d\boldsymbol\phi_x.\end{equation*}
Now, we look at a single $\boldsymbol\phi$:
\begin{equation*}\begin{aligned}
&\mathop{\mathlarger{\int}}_{\boldsymbol\phi_x} P\left(\boldsymbol\phi_{x}|\boldsymbol\alpha\right) \prod\limits_{i=1}^{\lambda_{x}} P\left(\boldsymbol z_{xi}|\boldsymbol\phi_{x}\right)d\boldsymbol\phi_x \\
&=\mathop{\mathlarger{\int}}_{\boldsymbol\phi_x} \frac{\Gamma\left(\sum_{k=1}^{K_1\cdots K_j}\alpha_k\right)}{\prod_{k=1}^{K_1\cdots K_j}\Gamma(\alpha_k)}\prod\limits_{k=1}^{K_1\cdots K_j}\phi_{xk}^{\alpha_k-1} \prod\limits_{i=1}^{\lambda_{x}} P\left(\boldsymbol z_{xi}|\boldsymbol\phi_{x}\right)d\boldsymbol\phi_x.
\end{aligned}\end{equation*}
Letting $n_x^k$ denote the count of topic(s) $k$ given independent variable $x$, we can express
\begin{equation*}\prod\limits_{i=1}^{\lambda_{x}} P\left(\boldsymbol z_{xi}|\boldsymbol\phi_{x}\right)=\prod\limits_{k=1}^{K_1\cdots K_j}\phi_{xk}^{n_x^k}.\end{equation*}
Thus, the $\boldsymbol\phi_x$ integral can be rewritten as
\begin{equation*}\begin{aligned}
&\mathop{\mathlarger{\int}}_{\boldsymbol\phi_x} \frac{\Gamma\left(\sum_{k=1}^{K_1\cdots K_j}\alpha_k\right)}{\prod_{k=1}^{K_1\cdots K_j}\Gamma(\alpha_k)}\prod\limits_{k=1}^{K_1\cdots K_j}\phi_{xk}^{\alpha_k-1} \prod\limits_{k=1}^{K_1\cdots K_j}\phi_{xk}^{n_x^k}d\boldsymbol\phi_x \\
&=\mathop{\mathlarger{\int}}_{\boldsymbol\phi_x} \frac{\Gamma\left(\sum_{k=1}^{K_1\cdots K_j}\alpha_k\right)}{\prod_{k=1}^{K_1\cdots K_j}\Gamma(\alpha_k)}\prod\limits_{k=1}^{K_1\cdots K_j}\phi_{xk}^{n_x^k+\alpha_k-1}d\boldsymbol\phi_x.
\end{aligned}\end{equation*}
According to the functional expression of the Dirichlet distribution,
\begin{equation*}\mathop{\mathlarger{\int}}_{\boldsymbol\phi_x} \frac{\Gamma\left(\sum_{k=1}^{K_1\cdots K_j}n_x^k+\alpha_k\right)}{\prod_{k=1}^{K_1\cdots K_j}\Gamma\left(n_x^k+\alpha_k\right)}\prod\limits_{k=1}^{K_1\cdots K_j}\phi_{xk}^{n_x^k+\alpha_k-1}d\boldsymbol\phi_x=1.\end{equation*}
We apply this equation to get rid of the integral, resulting in a fraction made up of products of Gamma functions,
\begin{equation*}\begin{aligned}
&\mathop{\mathlarger{\int}}_{\boldsymbol\phi_x} P\left(\boldsymbol\phi_{x}|\boldsymbol\alpha\right) \prod\limits_{i=1}^{\lambda_{x}} P\left(\boldsymbol z_{xi}|\boldsymbol\phi_{x}\right)d\boldsymbol\phi_x \\ =&\mathop{\mathlarger{\int}}_{\boldsymbol\phi_x} \frac{\Gamma\left(\sum_{k=1}^{K_1\cdots K_j}\alpha_k\right)}{\prod_{k=1}^{K_1\cdots K_j}\Gamma(\alpha_k)}\prod\limits_{k=1}^{K_1\cdots K_j}\phi_{xk}^{n_x^k+\alpha_k-1}d\boldsymbol\phi_x \\
=&\frac{\Gamma\left(\sum_{k=1}^{K_1\cdots K_j}\alpha_k\right)}{\prod_{k=1}^{K_1\cdots K_j}\Gamma(\alpha_k)} \frac{\prod_{k=1}^{K_1\cdots K_j}\Gamma\left(n_x^k+\alpha_k\right)}{\Gamma\left(\sum_{k=1}^{K_1\cdots K_j}n_x^k+\alpha_k\right)} \\
&\times \mathop{\mathlarger{\int}}_{\boldsymbol\phi_x} \frac{\Gamma\left(\sum_{k=1}^{K_1\cdots K_j}n_x^k+\alpha_k\right)}{\prod_{k=1}^{K_1\cdots K_j}\Gamma(n_x^k+\alpha_k)}\prod\limits_{k=1}^{K_1\cdots K_j}\phi_{xk}^{n_x^k+\alpha_k-1}d\boldsymbol\phi_x \\
=&\frac{\Gamma\left(\sum_{k=1}^{K_1\cdots K_j}\alpha_k\right)}{\prod_{k=1}^{K_1\cdots K_j}\Gamma(\alpha_k)} \frac{\prod_{k=1}^{K_1\cdots K_j}\Gamma\left(n_x^k+\alpha_k\right)}{\Gamma\left(\sum_{k=1}^{K_1\cdots K_j}n_x^k+\alpha_k\right)}.\end{aligned}\end{equation*}
Similarly, we derive the $\boldsymbol\psi$ part, letting $m^{(j)}_{h y}$ denote the count of dependent variable $y$ in the $j^{th}$ mode given topic $h$:
\begin{equation*}\begin{aligned}
\mathop{\mathlarger{\int}}_{\boldsymbol\psi}\prod\limits_{j=1}^p \prod\limits_{h=1}^{K_j}P\left(\boldsymbol\psi_{h}^{(j)}|\boldsymbol\beta\right) &\prod\limits_{i=1}^{\lambda_{x}} P\left(\boldsymbol y_{xi}|\boldsymbol\psi_{\boldsymbol z_{xi}}\right) d\boldsymbol\psi \\
=&\prod\limits_{j=1}^p \prod\limits_{h=1}^{K_j}\mathop{\mathlarger{\int}}_{\boldsymbol\psi_{h}^{(j)}}P\left(\boldsymbol\psi_{h}^{(j)}|\boldsymbol\beta\right) \prod\limits_{i=1}^{\lambda_{x}} P\left( y_{xi}^{(j)}|\boldsymbol\psi_{z_{xi}^{(j)}}^{(j)}\right) d\boldsymbol\psi_{h}^{(j)} \\
=&\prod\limits_{j=1}^p \prod\limits_{h=1}^{K_j}\mathop{\mathlarger{\int}}_{\boldsymbol\psi_{h}^{(j)}} \frac{\Gamma\left(\sum_{y=1}^{d_j}\beta_{y}^{(j)}\right)}{\prod_{y=1}^{d_j}\Gamma(\beta_{y}^{(j)})}\prod\limits_{y=1}^{d_j} \left(\psi_{z_{xi}^{(j)}y}^{(j)}\right)^{\beta_{y}^{(j)}-1} \\
&\times \prod\limits_{y=1}^{d_j} \left(\psi_{z_{xi}^{(j)}y}^{(j)}\right)^{m^{(j)}_{h y}} d\boldsymbol\psi_{h}^{(j)} \\
=&\prod\limits_{j=1}^p \prod\limits_{h=1}^{K_j}\mathop{\mathlarger{\int}}_{\boldsymbol\psi_{h}^{(j)}} \frac{\Gamma\left(\sum_{y=1}^{d_j}\beta_{y}^{(j)}\right)}{\prod_{y=1}^{d_j}\Gamma(\beta_{y}^{(j)})}\prod\limits_{y=1}^{d_j} \left(\psi_{z_{xi}^{(j)}y}^{(j)}\right)^{m^{(j)}_{h y}+\beta_{y}^{(j)}-1}d\boldsymbol\psi_{h}^{(j)} \\
=&\prod\limits_{j=1}^p \prod\limits_{h=1}^{K_j}\frac{\Gamma\left(\sum_{y=1}^{d_j}\beta_{y}^{(j)}\right)}{\prod_{y=1}^{d_j}\Gamma(\beta_{y}^{(j)})}\frac{\prod_{y=1}^{d_j}\Gamma\left(m^{(j)}_{h y}+\beta_{y}^{(j)}\right)}{\Gamma\left(\sum_{y=1}^{d_j}m^{(j)}_{h y}+\beta_{y}^{(j)}\right)}.
\end{aligned}\end{equation*}
By combining the expressions from the $\boldsymbol\phi$ and $\boldsymbol\psi$ parts, we obtain,
\begin{equation*}\begin{aligned}
P(\boldsymbol Y, \boldsymbol Z|\boldsymbol\alpha, \boldsymbol\beta)=
&\prod\limits_{x=1}^{d_0}\frac{\Gamma\left(\sum_{k=1}^{K_1\cdots K_j}\alpha_k\right)}{\prod_{k=1}^{K_1\cdots K_j}\Gamma(\alpha_k)} \frac{\prod_{k=1}^{K_1\cdots K_j}\Gamma\left(n_x^k+\alpha_k\right)}{\Gamma\left(\sum_{k=1}^{K_1\cdots K_j}n_x^k+\alpha_k\right)} \\
&\times \prod\limits_{j=1}^p \prod\limits_{h=1}^{K_j}\frac{\Gamma\left(\sum_{y=1}^{d_j}\beta_{y}^{(j)}\right)}{\prod_{y=1}^{d_j}\Gamma(\beta_{y})}\frac{\prod_{y=1}^{d_j}\Gamma\left(m^{(j)}_{h y}+\beta_{y}\right)}{\Gamma\left(\sum_{y=1}^{d_j}m^{(j)}_{h y}+\beta_{y}^{(j)}\right)}.
\end{aligned}\end{equation*} \\
Next, we need to derive an expression for the probability distribution of $\boldsymbol z_{i}^{(c)}$, which denotes the hidden variable(s) for the $i^{th}$ count in $x=c$, where $\boldsymbol y = \boldsymbol v$. Let a superscript $-ci$ denote the count, excluding the $i^{th}$ count in $x=c$. By Bayes' Theorem,
\begin{equation*}P(\boldsymbol z_{i}^{(c)}|\boldsymbol Z^{-ci},\boldsymbol Y, \boldsymbol\alpha,\boldsymbol\beta)=\frac{P(\boldsymbol z_{i}^{(x)},\boldsymbol Z^{-ci},\boldsymbol Y| \boldsymbol\alpha,\boldsymbol\beta)}{P(\boldsymbol Z^{-ci},\boldsymbol Y| \boldsymbol\alpha,\boldsymbol\beta)}.\end{equation*}
By applying this and dropping the denominator, we can express the probability distribution of $\boldsymbol z_{i}^{(c)}$ as proportional to the expression we derived above,
\begin{equation*}\begin{aligned}
&P(\boldsymbol z_{i}^{(c)}=\boldsymbol k|\boldsymbol Z^{-ci},\boldsymbol Y, \boldsymbol\alpha,\boldsymbol\beta)\propto P(\boldsymbol z_{i}^{(c)}=\boldsymbol k,\boldsymbol Z^{-ci},\boldsymbol Y| \boldsymbol\alpha,\boldsymbol\beta) \\
\propto&\left(\frac{\Gamma\left(\sum_{k=1}^{K_1\cdots K_j}\alpha_k\right)}{\prod_{k=1}^{K_1\cdots K_j}\Gamma(\alpha_k)}\right)^{d_0}\prod\limits_{x\neq c}\frac{\prod_{k=1}^{K_1\cdots K_j}\Gamma\left(n_x^k+\alpha_k\right)}{\Gamma\left(\sum_{k=1}^{K_1\cdots K_j}n_x^k+\alpha_k\right)} \\
&\times \prod\limits_{j=1}^p \left(\frac{\Gamma\left(\sum_{y=1}^{d_j}\beta_{y}^{(j)}\right)}{\prod_{y=1}^{d_j}\Gamma(\beta_{y}^{(j)})}\right)^{K_j}\prod\limits_{h=1}^{K_j}\prod\limits_{y\neq v_j}\Gamma\left(m^{(j)}_{h y}+\beta_{y}^{(j)}\right) \\
&\times \frac{\prod_{k=1}^{K_1\cdots K_j}\Gamma\left(n_c^k+\alpha_k\right)}{\Gamma\left(\sum_{k=1}^{K_1\cdots K_j}n_c^k+\alpha_k\right)}\prod\limits_{j=1}^p\prod\limits_{h=1}^{K_j} \frac{\Gamma\left(m^{(j)}_{h v_j}+\beta_{v_j}^{(j)}\right)}{\Gamma\left(\sum_{y=1}^{d_j}m^{(j)}_{h y}+\beta_{y}^{(j)}\right)}.\end{aligned}\end{equation*}
\normalsize
We simplify this expression by dropping multiplicative constants,
\begin{equation*}\begin{aligned}&\propto \frac{\prod_{k=1}^{K_1\cdots K_j}\Gamma\left(n_c^k+\alpha_k\right)}{\Gamma\left(\sum_{k=1}^{K_1\cdots K_j}n_c^k+\alpha_k\right)}\prod\limits_{j=1}^p\prod\limits_{h=1}^{K_j} \frac{\Gamma\left(m^{(j)}_{h v_j}+\beta_{v_j}^{(j)}\right)}{\Gamma\left(\sum_{y=1}^{d_j}m^{(j)}_{h y}+\beta_{y}^{(j)}\right)} \\
&\propto \prod\limits_{h=1}^{K_1\cdots K_j}\Gamma\left(n_c^h+\alpha_h\right)\prod\limits_{j=1}^p\prod\limits_{h=1}^{K_j} \frac{\Gamma\left(m^{(j)}_{h v_j}+\beta_{v_j}^{(j)}\right)}{\Gamma\left(\sum_{y=1}^{d_j}m^{(j)}_{h y}+\beta_{y}^{(j)}\right)}.\end{aligned}\end{equation*}
We now split this expression to obtain an $h$-independent summation, which can be dropped,
\begin{equation*}\begin{aligned}
\propto &\prod\limits_{h\neq \boldsymbol k}\Gamma\left(n_c^{h,-ci}+\alpha_h\right)\prod\limits_{j=1}^p\prod\limits_{h\neq k_j} \frac{\Gamma\left(m^{(j),-ci}_{h v_j}+\beta_{v_j}^{(j)}\right)}{\Gamma\left(\sum_{y=1}^{d_j}m^{(j),-ci}_{h y}+\beta_{y}^{(j)}\right)} \\
&\times \Gamma\left(n_c^{\boldsymbol k,-ci}+\alpha_{\boldsymbol k}+1\right)\prod\limits_{j=1}^p\frac{\Gamma\left(m^{(j),-ci}_{k_j v_j}+\beta_{v_j}^{(j)}+1\right)}{\Gamma\left(\sum_{y=1}^{d_j}m^{(j),-ci}_{k_j y}+\beta_{y}^{(j)}+1\right)} \\
\propto &\Gamma\left(n_c^{\boldsymbol k,-ci}+\alpha_{\boldsymbol k}\right)\prod\limits_{j=1}^p\frac{\Gamma\left(m^{(j),-ci}_{k_j v_j}+\beta_{v_j}^{(j)}\right)}{\Gamma\left(\sum_{y=1}^{d_j}m^{(j),-ci}_{k_j y}+\beta_{y}^{(j)}\right)} \\
&\times \left(n_c^{\boldsymbol k,-ci}+\alpha_{\boldsymbol k}\right)\prod\limits_{j=1}^p\frac{m^{(j),-ci}_{k_j v_j}+\beta_{v_j}^{(j)}}{\sum_{y=1}^{d_j}m^{(j),-ci}_{k_j y}+\beta_{y}^{(j)}} \\
\propto &\left(n_c^{\boldsymbol k,-ci}+\alpha_{\boldsymbol k}\right)\prod\limits_{j=1}^p\frac{m^{(j),-ci}_{k_j v_j}+\beta_{v_j}^{(j)}}{\sum_{y=1}^{d_j}m^{(j),-ci}_{k_j y}+\beta_{y}^{(j)}}.\end{aligned}\end{equation*}
\normalsize

\section{Generalizations to $p\ge3$}
\label{Generalizations}

\begin{figure}[t]
\centering
\subfigure[Hierarchical Model]{
\centering
\label{Hierarchical Model}
\hspace{10mm}
\begin{tikzpicture}[scale=.8, transform shape]
\tikzstyle{main}=[circle, minimum size =3mm, thick, draw =black!80, node distance = 4mm]
\tikzstyle{connect}=[-latex, thin]
\node[main] (1) {1};
\node[main] (2) [below left=8mm and 2mm of 1] {2};
\node[main] (3) [below right=8mm and 2mm of 1] {3};
\path (1) edge [connect] (2)
(1) edge [connect] (3);
\end{tikzpicture}
\hspace{10mm}}
\subfigure[Mixed Model]{
\centering
\label{Mixed Model}
\hspace{10mm}
\begin{tikzpicture}[scale=.8, transform shape]
\tikzstyle{main}=[circle, minimum size =3mm, thick, draw =black!80, node distance = 4mm]
\tikzstyle{connect}=[-latex, thin]
\node[main] (2) {2};
\node[main] (1) [below left=3mm and 6mm of 1] {1};
\node[main] (3) [below=8mm of 2] {3};
\path (2) edge [connect] (3);
\end{tikzpicture}
\hspace{10mm}
}
\caption{General $p$ Topic Model Examples ($p=3$)}
\end{figure}

For the independent topic model, we generalize to $p$ independent hierarchical models (such as trees), where topic tuples $\boldsymbol{T}_x$ are all possible combinations of $\boldsymbol{T}_x^{(1)},\ \cdots,$ and $\boldsymbol{T}_x^{(p)}$. For the hierarchical model, with general $p$, it is possible to have more complex dependence relations between modes. For example, Figure \ref{Hierarchical Model} shows a model where modes $2$ and $3$ depend on mode $1$ but do not directly depend on each other. To this end, we assume we are given a DAG representing the dependency structure of the $p$ modes. At each level, each customer moves through the DAG and selects a topic in each mode (ordered topologically) based on probability $\boldsymbol P_{\boldsymbol i}^{(\ell,m)}$, where $\ell$ is the level, $m$ is the mode, and $\boldsymbol i$ represents the topics of all parent modes. Similarly to the $p=2$ case, the tuples of topics visited at each level of the DAG or all possible tuples of elements in the topic lists for each mode correspond to the topic list for each $x$ (those topics with non-zero prevalence). Furthermore, it is possible to have a mixture of independent and hierarchical topic models; a specific mode or set of modes could be independent of the other modes. For example, Figure \ref{Mixed Model} presents a model where mode $1$ is independent of modes $2$ and $3$, while mode $3$ depends on mode $2$.

\section{Derivatives of Log-Likelihood}
\label{Derivatives}

We define Log-Likelihood $LL$ as the log of our model probability, found in Section \ref{Conditional HBTD}, \refp{hierarchical prob}.
\begin{equation*}\begin{aligned}
LL=&\log P(\boldsymbol Y, \boldsymbol Z, \boldsymbol\phi, \boldsymbol\psi, \boldsymbol T|\boldsymbol\alpha, \boldsymbol\beta,\boldsymbol\gamma) \\
=&\log P(\boldsymbol\psi|\boldsymbol\beta) + \log P(\boldsymbol\phi,\boldsymbol T|\boldsymbol\alpha, \boldsymbol\gamma) +\log P(\boldsymbol Y,\boldsymbol Z |\boldsymbol\phi, \boldsymbol\psi)
\end{aligned}\end{equation*}

First, we derive the derivative with respect to $\alpha_i$:
\begin{equation*}\begin{aligned}
\frac{\partial LL}{\partial \alpha_i}=&\frac{\partial}{\partial \alpha_i} \log P(\boldsymbol\phi,\boldsymbol T|\boldsymbol\alpha, \boldsymbol\gamma) \\
=&\frac{\partial}{\partial \alpha_i} \log\prod\limits_{x=1}^{d_0}P\left(\boldsymbol\phi_{x}|\boldsymbol\alpha, \boldsymbol T_{x}\right) P\left(\boldsymbol T_{x}|\boldsymbol\gamma\right) \\
=&\frac{\partial}{\partial \alpha_i}\sum\limits_{x=1}^{d_0} \log P\left(\boldsymbol\phi_{x}|\boldsymbol\alpha, \boldsymbol T_{x}\right) \\
=&\sum\limits_{x=1}^{d_0}\frac{1}{P\left(\boldsymbol\phi_{x}|\boldsymbol\alpha, \boldsymbol T_{x}\right)}\frac{\partial P\left(\boldsymbol\phi_{x}|\boldsymbol\alpha, \boldsymbol T_{x}\right)}{\partial \alpha_i} \\
=&\sum\limits_{x=1}^{d_0}\frac{B(\boldsymbol\alpha)}{\prod_{j=1}^{\boldsymbol K}\phi_{xj}^{\alpha_j-1}} \frac{\partial}{\partial \alpha_i} \left[\frac{1}{B(\boldsymbol\alpha)}\prod_{j=1}^{\boldsymbol K}\phi_{xj}^{\alpha_j-1}\right] \\
=&\sum\limits_{x=1}^{d_0}\left[B(\boldsymbol\alpha) \frac{\partial}{\partial \alpha_i} \frac{1}{B(\boldsymbol\alpha)} + \frac{\alpha_i-1}{\phi_{xi}}\right] \\
=&\sum\limits_{x=1}^{d_0}\left[\varPsi^{(0)}\left(\sum_{j=1}^{\boldsymbol K}\alpha_j\right)-\varPsi^{(0)}(\alpha_i) + \frac{\alpha_i-1}{\phi_{xi}}\right], \\
\end{aligned}\end{equation*}
where $B$ is the Beta function and $\varPsi^{(0)}$ is the Polygamma function.

Next, we similarly derive the derivative with respect to $\beta_i^{(j)}$:
\begin{equation*}\begin{aligned}
\frac{\partial LL}{\partial \beta_i^{(j)}}=&\frac{\partial}{\partial \beta_i^{(j)}}\log P(\boldsymbol\psi|\boldsymbol\beta)=\sum\limits_{h_j=1}^{K_j}\log P\left(\boldsymbol\psi_{h_j}^{(j)}\big|\boldsymbol\beta^{(j)}\right) \\
=&\sum\limits_{h_j=1}^{K_j}\left[\varPsi^{(0)}\left(\sum_{y=1}^{d_j}\beta_y^{(j)}\right)-\varPsi^{(0)}(\beta_i^{(j)}) + \frac{\beta_i^{(j)}-1}{\psi_{h_ji}^{(j)}}\right]
\end{aligned}\end{equation*}

\section{Conditions Proofs}
\label{Proofs}

First, we prove Theorem \ref{Loose Partition}, showing that the parameters of the Dirichlet distributions in our PAM model are symmetric.

\begin{proof}
First, we look at a single Dirichlet distribution, i.e., a single-node PAM model, and show that the partition property applies if and only if the parameters are symmetric. Without loss of generality, assume the parameters of the Dirichlet distribution $\boldsymbol\gamma=(\gamma_1,\cdots,\gamma_K)$ are ordered such that $\gamma_1 \le \cdots \le \gamma_K$, where $K$ is the number of topics. Also, by PAM $\boldsymbol\theta=(\theta_1,\cdots,\theta_K)\sim\text{Dir}(\boldsymbol\gamma)$.

First, we show sufficiency. To this end, assume $\gamma_p=\gamma_q:=\gamma$ for any $p,\ q$. Here, both the probability distribution and expectation of all $\theta$'s are equal (this is clear from examining the probability distribution function of the Dirichlet distribution). We denote the probability of $m_p$ people to topic $p$ and $m_q$ people to topic $q$, along with our topics assignment, as $P(m_p, m_q, -)$. Thus this probability,
\begin{equation*}\begin{aligned} P(m_p,m_q,-) & \propto \int_0^1 \int_0^1 P(m_p,m_q|\theta_p,\theta_q) P(\theta_p,\theta_q) d\theta_p d\theta_q \\
& \propto \int_0^1 \int_0^1 \theta_p^{m_p+\gamma-1}\theta_q^{m_q+\gamma-1} d\theta_pd\theta_q \\
& =\int_0^1 \theta^{m_p+\gamma-1} d\theta \int_0^1 \theta^{m_q+\gamma-1} d\theta,\end{aligned}\end{equation*}
is equal to the probability of assigning $m_q$ people to topic $p$ and $m_p$ people to topic $q$, \\
\begin{equation*}P(m_q,m_p,-) \propto\int_0^1 \theta^{m_p+\gamma-1} d\theta \int_0^1 \theta^{m_q+\gamma-1} d\theta.\end{equation*} \\
Similarly, since the probability distributions over all $\theta$'s are the same, the probability of assigning $m$ people to topics 1 through $K$ is equal to the probability of assigning any permutation of $m$ people to topics 1 through $k$. Thus the partition property holds.

Next, we show necessity. To this end, assume $\gamma_p<\gamma_q$ (for some $p$ and $q$). The probability of assigning $m_p$ people to topic $p$ and $m_q$ people to topic $q$,
\begin{equation*}
P(m_p,m_q,-) \propto \int_0^1\int_0^1 \theta_p^{m_p+\gamma_p-1}\theta_q^{m_q+\gamma_q-1}d\theta_pd\theta_q=\frac{1}{(m_p+\gamma_p)(m_q+\gamma_q)},
\end{equation*}
is not equal to the probability of assigning $m_q$ people to topic $p$ and $m_q$ people to topic $q$, \\
\begin{equation*}P(m_q,m_p,-) \propto\frac{1}{(m_p+\gamma_q)(m_q+\gamma_p)},\end{equation*} \\
for all $\ m_p$ and $m_q$. If $m_p<m_q$, then elementary algebra shows $(m_p+\gamma_p)(m_q+\gamma_q)<(m_p+\gamma_q)(m_q+\gamma_p)$ and if $m_p>m_q$, then $(m_p+\gamma_p)(m_q+\gamma_q)>(m_p+\gamma_q)(m_q+\gamma_p)$. Thus, the partition property does not hold.

If and only if the partition property holds for single nodes of the PAM model, it is possible to re-arrange topics in both modes (i.e., the loose partition property applies).
\end{proof}

Next, we prove Lemma \ref{function form lemma}, showing that for the given properties to hold, then $\boldsymbol\xi$ and $\boldsymbol\theta$ must be of the given form.

\begin{proof}
For such a model, $\xi_i=\xi_j$ if and only if $\rho_{i(\cdot)}=\rho_{j(\cdot)}$ and $\theta_{ki}=\theta_{kj}$ if and only if $\rho_{ki}=\rho_{kj}$. This is sufficient because of the strong partition property and necessary because of the rich-get-richer property. By the strong partition property and the chain rule, $\xi_i \theta_{ij}=\xi_k \theta_{kl}$ if $\rho_{ij}=\rho_{kl}$. Also, because of the rich-get-richer and exchangeability properties, we can express $\xi_i\propto\begin{cases} f\left(\rho_{i(\cdot)}\right), & \rho_{i(\cdot)}>0 \\ \gamma_0(K_1), & \rho_{i(\cdot)}=0\end{cases}$ and  $\theta_{ij}\propto\begin{cases} g_i(\rho_{ij}), & \rho_{ij}>0 \\ \gamma_i(K_2), & \rho_{ij}=0\end{cases}$.
\end{proof}

Then, we prove Lemma \ref{linear lemma}, showing that the functions must be linear.

\begin{proof} Suppose we want to assign $x$ people to topic 1 and one person to topic 2. One way (case one) to do this would be to assign all $x$ people to topic 1, then one person to topic 2. Another way (case two) to do this would be to assign $x-1$ people to topic 1, then one person to topic 2, then one more customer to topic 1. The probabilities of these cases can be expressed as:
\begin{equation*}\begin{aligned}
&P(\text{case one}) \\
= &\frac{f(1)}{\gamma_0(1)+f(1)}\cdots\frac{f(x-2)}{\gamma_0(1)+f(x-2)}\frac{f(x-1)}{\gamma_0(1)+f(x-1)}\frac{\gamma_0(1)}{\gamma_0(1)+f(x)} \\
&P(\text{case two}) \\
= &\frac{f(1)}{\gamma_0(1)+f(1)}\cdots\frac{f(x-2)}{\gamma_0(1)+f(x-2)}\frac{\gamma_0(1)}{\gamma_0(1)+f(x-1)}\frac{f(x-1)}{\gamma_0(2)+f(x-1)+f(1)}
\end{aligned}\end{equation*}
Thus, the differences in the probabilities are that the first case has $\gamma_0(1)+f(x)$ in the last denominator while the second case has $\gamma_0(2)+f(x-1)+f(1)$ in the last denominator. If the exchangeability property applies, we have $P(\text{case one})=P(\text{case two})$ and in turn $\gamma_0(1)+f(x)=\gamma_0(2)+f(x-1)+f(1)$. Since this must apply for all $x$, by induction we have $f(x)=x f(1)+\gamma_0(2)-\gamma_0(1)$. Thus $f$ is linear.

In general, $\xi_i\propto\begin{cases} \rho_{i(\cdot)}-\gamma_{02}, & \rho_{i(\cdot)}>0 \\ \gamma_{01}+\gamma_{02} K_1, & \rho_{i(\cdot)}=0\end{cases}$ and $\theta_{ij}\propto\begin{cases} \rho_{ij}-\gamma_{i2}, & \rho_{ij}>0 \\ \gamma_{i1}+\gamma_{i2} K_2, & \rho_{ij}=0\end{cases}$. Note that this is a generalized nCRP.
\end{proof}

Finally, we prove Theorem \ref{DNE}, showing that such a model does not exist.

\begin{proof}
Given the forms of $\boldsymbol\xi$ and $\boldsymbol\theta$ from Lemma \ref{linear lemma}, suppose we take two elements $(i,j)\neq(m,n),\ i\neq m$ and swap them within $\boldsymbol\rho$. Then the ratio $\nu$ of the original probability with the swapped probability is:
\begin{equation*}\begin{aligned}\omega(\rho_{ij},\rho_{mn})&:= &\frac{\Gamma(\rho_{m(\cdot)}-\gamma_{02})\Gamma(\rho_{i(\cdot)}-\gamma_{02})}{\Gamma(\rho_{m(\cdot)}-\rho_{mn}+\rho_{ij}-\gamma_{02})\Gamma(\rho_{i(\cdot)}-\rho_{ij}+\rho_{mn}-\gamma_{02})} \\
\nu &= &\omega(\rho_{ij},\rho_{mn})\frac{\Gamma(\rho_{ij}-\gamma_{i2})\Gamma(\rho_{mn}-\gamma_{m2})}{\Gamma(\rho_{ij}-\gamma_{m2})\Gamma(\rho_{mn}-\gamma_{i2})}.\end{aligned}\end{equation*} \\
If the strict partition property is satisfied, then $\nu=1$ for all $\rho_{ij},\ \rho_{mn},\ \rho_{i(\cdot)},$ and $\rho_{m(\cdot)}$. Note that $\frac{\Gamma(\rho_{ij}-\gamma_{i2})\Gamma(\rho_{mn}-\gamma_{m2})}{\Gamma(\rho_{ij}-\gamma_{m2})\Gamma(\rho_{mn}-\gamma_{i2})}=1$ for all $\rho_{ij}$ and $\rho_{mn}$ if and only if $\gamma_{i2}=\gamma_{m2}$, however $\omega(\rho_{ij},\rho_{mn})\neq 1$ for all $\rho_{ij},\ \rho_{mn},\ \rho_{i(\cdot)},$ and $\rho_{m(\cdot)}$. For example, if $\rho_{ij}=1,\ \rho_{mn}=2,\ \rho_{i(\cdot)}=3,\ \rho_{m(\cdot)}=3,$ and $\gamma_{02}=0$, then $\omega(\rho_{ij},\rho_{mn})=\frac{2!2!}{1!3!}=\frac{2}{3}\neq 1$. This shows that $\nu=1$ is not always possible.
\end{proof}

\section{Model Variations}
\label{variations}

In our efforts to boost the performance of our models, we implemented and developed several variations and modifications to our original methodology, including:

\begin{itemize}
    \item \textbf{Different seeds:} Train the model using ten different random number generator seeds, picking the best model based on either log-likelihood or coherence.
    \item \textbf{Keep best:} Check the log-likelihood or coherence during training every ten iterates, saving the best model rather than the final one. Another variation on this method is resetting to the current best model after each check.
    \item \textbf{MAP estimate:} When using the collapsed sampler, use the maximum a posteriori (MAP) estimate of the auxiliary matrices and core tensors rather than performing one non-collapsed iterate.
    \item \textbf{Adjust the number of counts:} For example, double the counts or set all non-zero counts to one.
    \item \textbf{Sparse cutoff:} Set a cutoff value, below which all proportional probabilities are set to zero.
    \item \textbf{Initialization:} Initialize the auxiliary matrices with those trained using a different method.
    \item \textbf{Relative number of topic model iterates to Bayesian Tucker:} Do two Bayesian Tucker sample iterations for every draw from the hierarchical topic model, or vice-versa.
    \item \textbf{Set a topics goal:} Decide on an ideal number of topics. Adjust $\gamma$ during training to reach that goal. The formula we used was:
    $$\gamma_{\text{new}}=\gamma_{\text{old}} * \max\left\{\min\left\{\left(\frac{\text{topics goal}}{\text{\# of topics}}\right)^{1/\prod_{i=1}^p (L_i-1)},2\right\},0.5\right\}.$$
    \item \textbf{Exponential weighting:} Apply an exponential weight to the relative probabilities in the Collapsed Gibbs Sampler (Algorithm \ref{Collapsed Bayesian Tucker Decomposition Gibbs Sampler}) as such:
    $$P\left(\boldsymbol z_{i}^{(x)}=\boldsymbol{k}|-\right)\propto\left(n_x^{\boldsymbol k,-xi}+\alpha_{\boldsymbol k}\right)\prod\limits_{j=1}^p\left[\frac{m^{(j),-xi}_{k_j y_j}+\beta_{y_j}^{(j)}}{\sum_{y=1}^{d_j}m^{(j),-xi}_{k_j y}+\beta_{y}^{(j)}}\right]^{w_p},$$
    where $w_p$ is the weight for mode $p$. This allows us to adjust the relative variance, i.e., the uniformity of each mode.
\end{itemize}

For our coherence measure experiments (Section \ref{Coherence} and Appendix \ref{Reuters}), we used different seeds, keep best, MAP estimate, topics goals, and exponential weighting. We used keep best for our classification models (Appendix \ref{Classsification}). While we tried other methods above, they did not perform as well. We did not use the above modifications for the likelihood models (Appendix \ref{Likelihood}).

\section{Experiments}
\label{Experiments}

Code available at \url{https://github.com/ars2240/asdHBTucker}. The decomposition algorithm was mainly written in MATLAB, using Tensor Toolbox \citep{bader_kolda_2021}. Sampling functions were written in C to improve run time. The preprocessing for the genetic data was done in R. The text data was preprocessed in Python, using Gensim \citep{rehurek2011gensim} and the Natural Language Toolkit (NLTK) \citep{bird2009natural}. The non-hierarchical CP decompositions were also done in Python, using TensorLy \citep{tensorly}. Experiments were run on a 3.1 GHz Dual-Core Intel Core i5 processor with 16 GB 2133 MHz LPDDR3 memory.

\subsection{Data Sets}
\label{A. data sets}

\begin{table}[t]
\caption{Cancer types and TCGA designation}
\label{cancer designation}
\centering
\small
\begin{tabular}{c|Y{95mm}}
    Cancer type & TCGA designation  \\ \hline
    Breast & Breast Invasive Carcinoma (BRCA) \\
    Lung & Lung Squamous Cell Carcinoma (LUSC) or Lung Adenocarcinoma (LUAD) \\
    Prostate & Prostate Adenocarcinoma (PRAD) \\
    Colorectal & Colon Adenocarcinoma (COAD) or Rectum Adenocarcinoma (READ)
\end{tabular}
\end{table}

The cancer data contains 3,037 patients. The Cancer Genome Atlas (TCGA) designations we considered for each type of cancer are given in Table \ref{cancer designation}. This data set contained 1,044 patients with breast cancer, 1,066 patients with lung cancer, 494 patients with prostate cancer, and 433 patients with colorectal cancer. After linking with the Reactome pathways, we were left with counts of variants on 7,846 genes and 1,678 pathways. This includes single-nucleotide polymorphisms and other mutations. We do not consider gene expression data.

Once we similarly linked the ASD genetic variants to the pathways, our data set contains 3,408 patients (half diagnosed with ASD), 7,211 genetic variants, and 1,413 pathways.

The Reuters data contains 5,501 articles, 8,820 phrases, and 6,837 unique words.

We split each data set into a 30\% held-out test set and performed 10-fold cross-validation (CV) on the remaining training/validation data.

\subsection{Coherence Measure}
\label{Coherence Measure}

To utilize this coherence measure in our context, we made some modifications: 1) our metric is intrinsic and does not utilize an external corpus to determine the gene or pathway probabilities and co-occurrence probabilities; and 2) we determined co-occurrence as having a variant on a pair of genes or pathways, not accounting for sequences as in many natural language processing examples.

\subsection{Coherence Models}
\label{Coherence Models}

To pre-process our data, we removed genes that appeared in fewer than 200 or more than 2,000 patients and words that appeared in fewer than 200 or more than 2,000 documents, eliminating rare and common genes or words. We also removed phrases that were present in fewer than 10 articles. Words without an assigned phrase were then assigned to a single-word phrase corresponding to the given word.
We modified \citepos{ding2020more} implementation to process the R8 data into the article, phrase, and word tensor structure. The CP TensorLy model was trained on the entire training data set, and then split into folds; while the other models were split into folds before fitting the decompositions. Our coherence measures examined the top 5 genes, pathways, sentences, or words in each topic.

Bayesian cancer and ASD models had a topic goal of 500, and R8 models had a topic goal of 50 (see Appendix \ref{variations}), while the CP TensorLy model had 200 topics for all data sets. The hierarchical models used three levels.

The HBT models used the independent trees hierarchy. We trained PAM-based models on the Cancer data and found the models slightly under-performed the independent trees models. The best PAM-based model had ten topics per level, used Genes as a dominant mode and Cartesian topics, and had a PMI genes coherence of 16.35 and PMI pathways coherence of 9.55.

The hLDA models used a modified version of the independent tree HBT model with only two modes. The CP tree model also used a modified independent tree HBT model, with a single hierarchical tree and restricting the core tensor to be diagonal. The CP TensorLy model uses the alternating-least squares method to decompose the tensor \citep{JMLR:v20:18-277}.

For the R8 models with bad phrases removed, we removed words and phrases that infrequently showed up in topics.

Experimenting on several HBT models showed that the model log-likelihood increased for the first 30-40 samples, then leveled off (with some random fluctuation). Based on this, we chose to run each model for 100 iterations, checking every 10 iterations and keeping the best model. We also used 10 different random seeds.

\subsection{Reuters Experiments}
\label{Reuters}

We also applied a similar structure from our genetic models to natural language processing.
Incorporating phrases as another mode may improve the grouping of articles. The hierarchical structure here would be: words make up phrases, which make up articles. We looked at the eight largest classes (earnings, acquisitions, money - foreign exchange, grain, crude, trade, interest, and shipping) in the Reuters-21578 data set (denoted R8) \citep{Dua:2019}. We used SpaCy \citep{spacy2} to group words into phrases (noun chunks). This gives us a count of words in each phrase in each article.

\begin{figure}[t]
    \centering
    \includegraphics[scale=0.94]{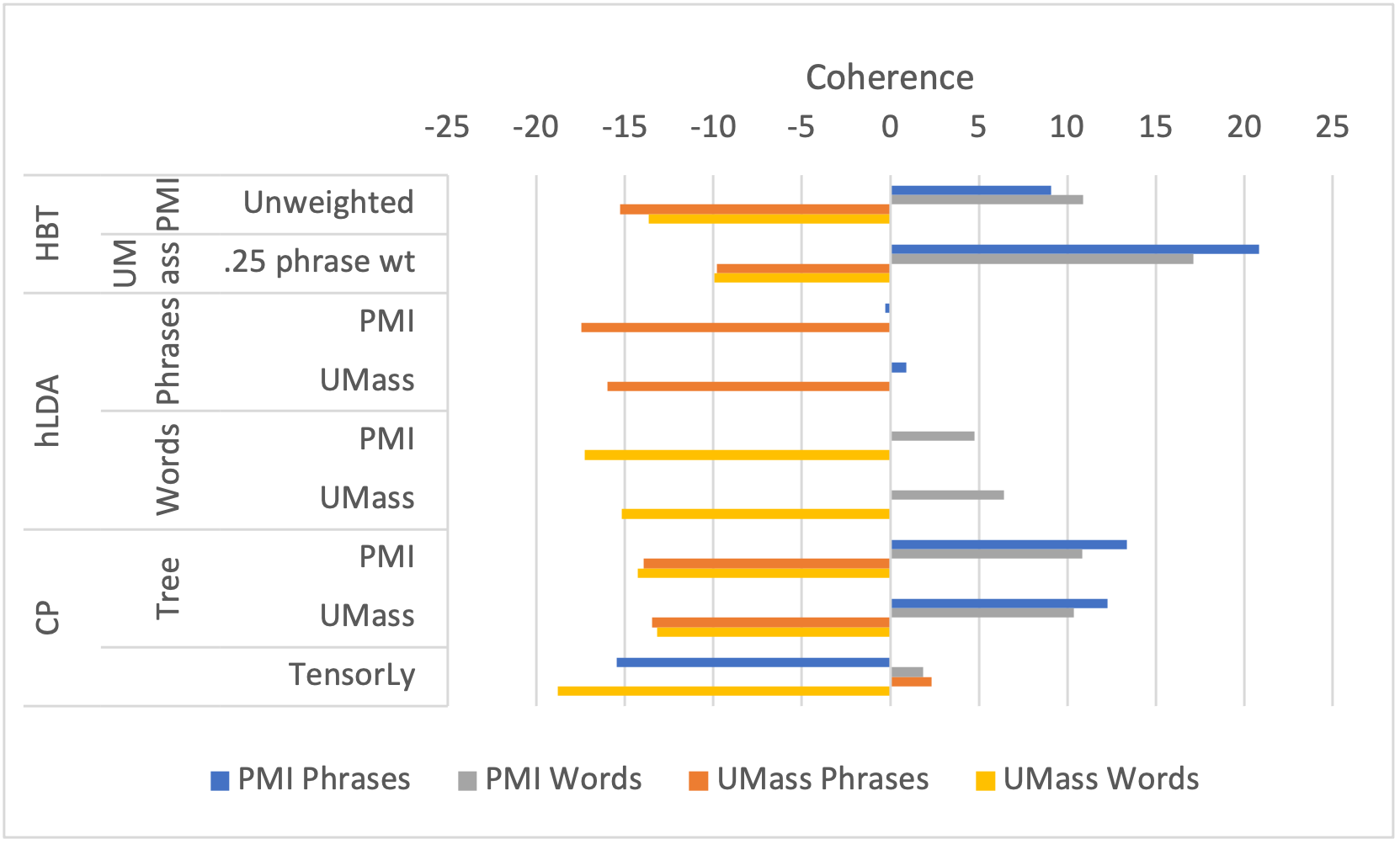}
    \caption{Independent Tree HBT models are more coherent than comparable models on R8 data. Our best HBT models outperformed the hLDA, CP tree, and CP TensorLy baselines on three-of-four coherence measures (using the mean over ten cross-validation folds).}
    \label{R8 Coherence}
\end{figure}

\begin{table*}[t]
\caption{Mean Validation Coherence of R8 Models with Bad Phrases Removed}
\label{R8 Coherence Bad Phrases}
\centering
\small
\begin{tabular}{l|r|r|r|r}
& \multicolumn{2}{c|}{Phrases} & \multicolumn{2}{c}{Words} \\ \cline{2-5}
Model & PMI & UMass & PMI & UMass \\ \hline \hline
HBT (PMI) & 11.80 & -12.91 & 12.38 & -12.37 \\ \hline
HBT (UMass) & \textbf{12.61} & -11.19 & \textbf{13.39} & \textbf{-10.79} \\ \hline \hline
CP (TensorLy) & -15.47 & \textbf{-6.64} & -4.00 & -20.04 \\
\end{tabular}
\end{table*}

While our best HBT model outperformed the baseline models on three-of-four coherence measures on the R8 data, it under-performed on UMass phrase coherence. Figure \ref{R8 Coherence} shows that the HBT model using UMass coherence and .25 phrase weighting outperformed the CP model using PMI coherence, the best baseline model, on PMI phrase and word coherence by 56.07\% and 57.84\% respectively. This HBT model also outperformed the CP model using UMass coherence by 24.54\% on UMass word coherence. However, while the CP TensorLy model performed worst on these three coherence measures, it outperformed our best model on UMass phrase coherence (2.35 to -9.79).
After examining this data set further, we removed words and phrases that frequently showed up in topics and had few co-occurrences with other words in the topic. While we tried various word sets and were able to make improvements, due to the Bayesian nature of the algorithms, words with fewer co-occurrences are more likely to be placed in the same topic than in a deterministic model. Table \ref{R8 Coherence Bad Phrases} shows that while these changes narrowed the gap in performance on UMass phrase coherence, the CP TensorLy model still outperformed our best HBT model on this measure.

\begin{figure*}[t]
    \centering
    \subfigure{
        \centering
        \includegraphics[scale=0.5]{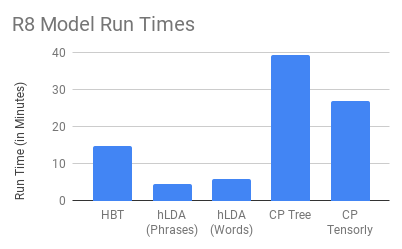}
    }
    \caption{Time to train a single decomposition model on R8 data.}
    \label{Run Time R8}
\end{figure*}

Figure \ref{Run Time R8} shows that while our algorithm takes longer to train on the R8 data than the ASD and Cancer data sets, it still significantly outperforms the CP models, supporting our conclusions in Section \ref{Complexity}.

\subsection{Classification}
\label{Classsification}

\begin{figure}[t]
    \centering
    \includegraphics[scale=0.94]{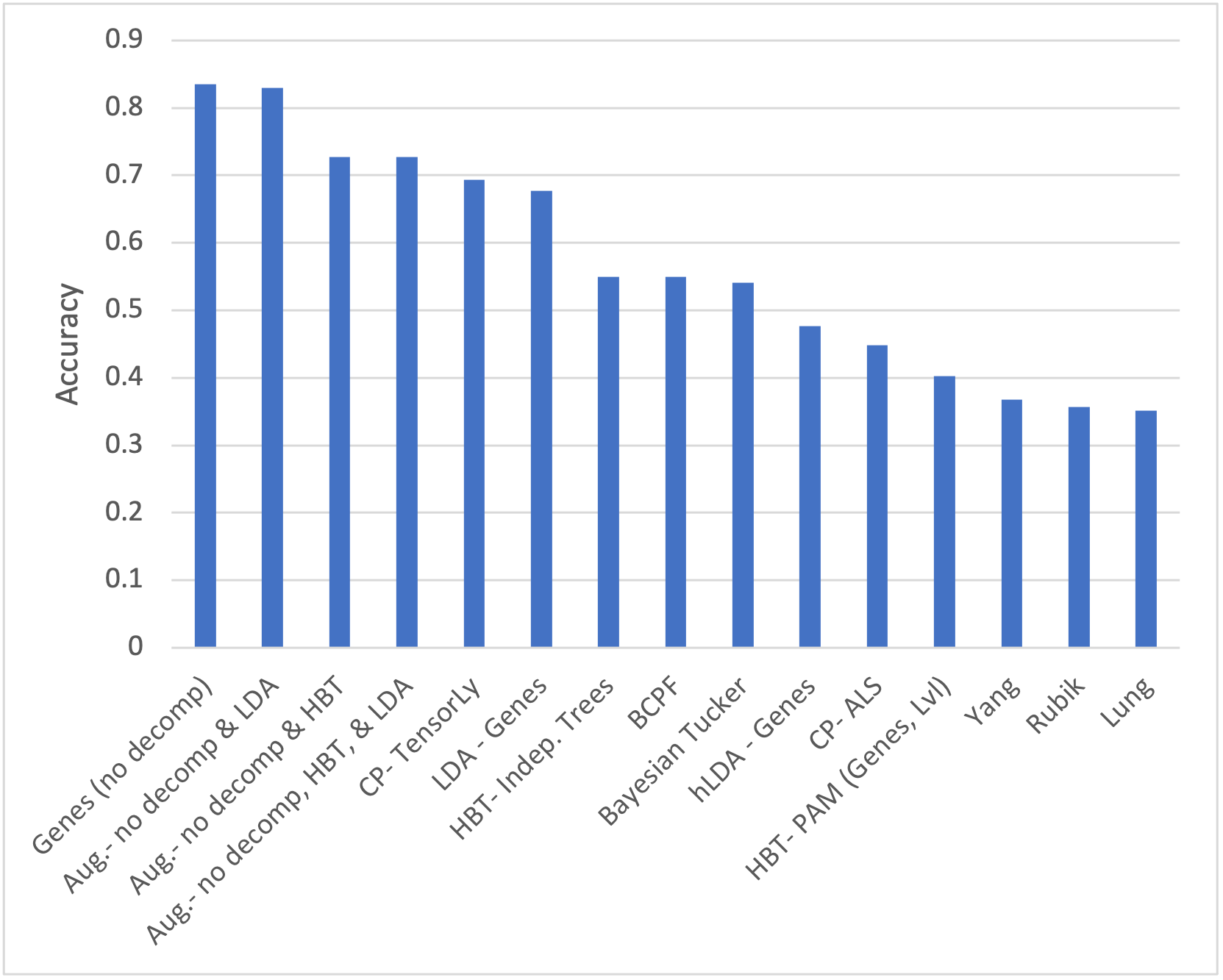}
    \caption{Genes data (without decomposition) is more accurate than decomposition and augmented models on cancer data (using the mean over ten cross-validation folds).}
    \label{Cancer Accuracy}
\end{figure}

We trained logistic regression models using patient groups from the decomposition models to predict each patient's cancer diagnosis. Figure \ref{Cancer Accuracy} shows the accuracy of these models, and Appendix \ref{Classification Models} details our implementation. Unfortunately, none of the models we trained, including the baselines, outperformed classifying using the genetic variant counts (i.e., not using a decomposition model), which had a mean validation accuracy of 83.49\%. Augmenting the genetic variant data with our decomposition models did not improve the accuracy of our predictions. Of the decomposition models, CP TensorLy performed best (69.33\% accuracy). The other CP models, \citepos{ZhaoZC14} Bayesian CP Factorization (BCPF) (54.94\% accuracy) and \citepos{bader_kolda_2021} alternating least squares (ALS) CP (44.78\% accuracy) performed significantly worse. The second-best decomposition model was LDA on the genetic variants (67.74\% accuracy). On the other hand, our genes-based hLDA model had a 47.68\% accuracy. Our independent trees HBT model and our non-hierarchical Bayesian Tucker model performed about the same, with accuracies of 54.98\% and 54.14\%, respectively. Our PAM-based HBT model, using genes as the dominant mode and the level method, performed worse (40.31\% accuracy). \citepos{yang_dunson_2016} Bayesian Conditional Tensor factorization (36.74\% accuracy) and \citepos{Wang2015} Rubik model (35.70\% accuracy) only slightly outperformed diagnosing all patients with Lung cancer, the most prevalent class, at 35.10\% prevalency.

\subsection{Classification Models}
\label{Classification Models}

For the CP TensorLy model, we removed genes that appeared in fewer than 200 or more than 2,000 patients and used 200 topics. For the LDA model, we modified \citepos{NIPS2010_3902} variational Bayes implementation using 40 topics. For our independent trees HBT model, we removed genes that appeared in fewer than 400 or more than 1,000 patients and used three levels, $\gamma=0.1$, and our "keep best" methodology. For our Bayesian Tucker model, we removed genes that appeared in fewer than 200 or more than 2,000 patients and used 10 topics on each mode and our "keep best" methodology. For our hLDA model, we used two levels and $\gamma=0.1$. For our CP ALS model, we used 25 topics. For our PAM-based HBT, we removed genes that appeared in fewer than 400 or more than 1,000 patients and used three levels and 10 topics per level. For the Rubik model, we used 5 topics.

\subsection{Likelihood}
\label{Likelihood}

One issue with comparing likelihoods across models is that the probabilities from Section \ref{Models} are not comparable due to differing hierarchical model structures. To compute the held-out likelihood, we would need to sum over or integrate out our hierarchical model variables, which do not have a closed-form solution.

To solve this problem, we use a non-parametric likelihood estimate, similar to \cite{Li:2006:PAD:1143844.1143917} and based on empirical likelihood \citep{Diggle}. First, we randomly generate one thousand patients using the trained generative process. Then, we compute the probabilities of a held-out test or validation patients as a mixture of the generated patients. Unlike other likelihood measures, this method is stable, easy to compute, and yields values that are comparable across models.

We trained HBT decomposition models using various hierarchical models and computed the mean validation log-likelihood (over the ten CV folds, using the above methodology). Each hierarchical model was trained with varying levels $L\in\{2,3,4,5\}$. Additionally, we trained the independent trees model with three CRP hyperparameters $\gamma\in\{0.5,1,2\}$. For the PAM-based model, we compared each choice of dominant mode (genes or pathways) and topic set composition (Cartesian or level set method, see Section \ref{Conditional HBTD} for definitions). We also used varying topics per level $\tau\in\{10,25,50\}$.

\begin{figure}[t]
    \centering
    \includegraphics[scale=0.125]{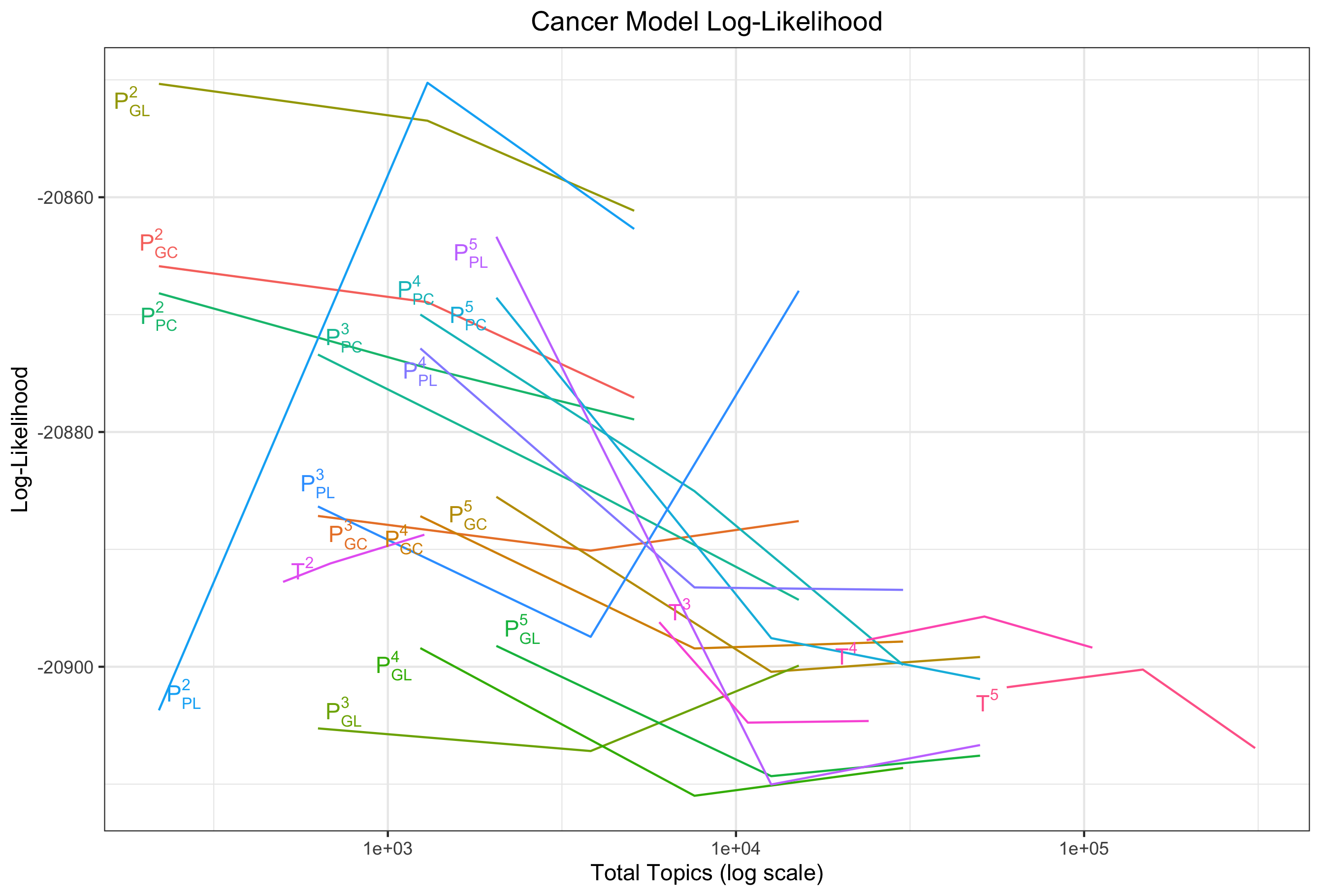}
    \caption{Each line has a coded label: The large letter indicates if the point is from the independent trees model ("T") or the PAM-based model ("P"). The superscript indicates the number of levels in the model. The subscript (for the PAM-based models) indicates the dominant mode, genes ("G") or pathways ("P"), followed by the topic set composition method, Cartesian ("C") or level set ("L"). Every label (combination of model type, number of levels, dominant mode, and topic set composition, if applicable) has the same color.}
    \label{Cancer LL Plot}
\end{figure}

We plotted the results in Figure \ref{Cancer LL Plot} and gave the values in Table \ref{Cancer LL Table}. This figure depicts the mean validation log-likelihood (over the ten CV folds) for hierarchical models trained on the cancer data set, comparing it to the total number of topics (product of the number of gene and pathway topics) for the model. In the case of the independent trees model, the number of topics is an average over the CV folds. Though the CRP hyperparameter ($\gamma$) and number of topics per level ($\tau$) are not displayed on the plot, they dictate the number of total topics.

Figure \ref{Cancer LL Plot} shows that the log-likelihood for models with fewer topics is often larger than those with more topics (14 of 20 lines peak at the fewest topics), indicating that some models are overfitted. We observed that some PAM-based models outperformed the independent trees models, which we hoped would be the case given the inherent hierarchical structure between genes and pathways. Still, there was no clear reason for which models were better than others (other than the total number of topics). The PAM-based models with level-set topic composition are the most sensitive to changes in the other parameters (the standard deviation of the log-likelihoods of all such models with genes as the dominant mode was 23.15 and 19.86 for pathway-dominant models). The independent tree models were the least sensitive (with a standard deviation of 5.68). The PAM-based models with Cartesian topic composition were in between (with a standard deviation of 11.52 for gene-dominant models and 12.56 for pathway-dominant models). However, the differences between these models are well within the margin of error. The standard deviation in log-likelihood over the CV folds for each model is about 2,090 (or 10\% the log-likelihood).

Table \ref{Cancer LL Table} gives the hierarchical topic model (independent trees or PAM-based), dominant mode (for PAM-based model, genes or pathways), whether the topic sets are created using the Cartesian or level method (See Section \ref{Conditional HBTD} for definitions), the value of $\gamma$ used (hyperparameter in CRP for the independent trees model, a uniform Dirichlet prior was used in PAM), the topics per level $\tau$ (for the PAM-based model), the number of hierarchical levels, the mean validation log-likelihood (over the 10-fold CV, computed using the method described in Section \ref{Likelihood}), the standard deviation of the log-likelihood (over the 10-fold CV), the number of gene and pathway topics (or mean number of topics across the CV folds in the case of the independent trees model), and the total number of topics (the product of the number of gene and pathway topics). This data is presented in Section \ref{Likelihood}.

\begin{table*}[t]
\caption{Cancer Log-Likelihood}
\label{Cancer LL Table}
\tiny
\centering
\begin{tabular}{l|c|c|c|c|c|c|c|c|c}
Topic & Dominant & & & & & & Gene & Pathway & Total   \\
Model  & Mode & Topic Set & $\gamma$/$\tau$ & Levels & Mean  & StDev & Topics & Topics & Topics  \\\hline
Trees &               & Cart. & 0.5  & 2      & -20893 & 2090 & 22        & 22           & 500        \\
Trees &               & Cart. & 1 & 2      & -20891 & 2092 & 27        & 26           & 684        \\
Trees &               & Cart. & 2     & 2      & -20889 & 2094 & 35        & 37           & 1273      \\
Trees &               & Cart. & 0.5   & 3      & -20896 & 2088 & 78        & 78           & 6021      \\
Trees &               & Cart. & 1     & 3      & -20905 & 2093 & 105       & 103          & 10820     \\
Trees &               & Cart. & 2     & 3      & -20905 & 2093 & 159       & 151          & 24055     \\
Trees &               & Cart. & 0.5   & 4      & -20898 & 2090 & 157       & 151          & 23729     \\
Trees &               & Cart. & 1     & 4      & -20896 & 2095 & 226       & 229          & 51709     \\
Trees &               & Cart. & 2     & 4      & -20898 & 2089 & 323       & 327          & 105599    \\
Trees &               & Cart. & 0.5   & 5      & -20902 & 2093 & 257       & 234          & 59938     \\
Trees &               & Cart. & 1     & 5      & -20900 & 2093 & 377       & 391          & 147520    \\
Trees &               & Cart. & 2     & 5      & -20907 & 2091 & 551       & 563          & 309852    \\
PAM               & Genes         & Cart. & 10               & 2      & -20866 & 2086 & 11          & 20             & 220          \\
PAM               & Genes         & Cart. & 25               & 2      & -20869 & 2093 & 26          & 50             & 1300         \\
PAM               & Genes         & Cart. & 50               & 2      & -20877 & 2089 & 51          & 100            & 5100         \\
PAM               & Genes         & Cart. & 10               & 3      & -20887 & 2094 & 21          & 30             & 630          \\
PAM               & Genes         & Cart. & 25               & 3      & -20890 & 2091 & 51          & 75             & 3825         \\
PAM               & Genes         & Cart. & 50               & 3      & -20888 & 2091 & 101         & 150            & 15150        \\
PAM               & Genes         & Cart. & 10               & 4      & -20887 & 2093 & 31          & 40             & 1240         \\
PAM               & Genes         & Cart. & 25               & 4      & -20898 & 2094 & 76          & 100            & 7600         \\
PAM               & Genes         & Cart. & 50               & 4      & -20898 & 2092 & 151         & 200            & 30200        \\
PAM               & Genes         & Cart. & 10               & 5      & -20886 & 2085 & 41          & 50             & 2050         \\
PAM               & Genes         & Cart. & 25               & 5      & -20900 & 2090 & 101         & 125            & 12625        \\
PAM               & Genes         & Cart. & 50               & 5      & -20899 & 2086 & 201         & 250            & 50250        \\
PAM               & Genes         & Level     & 10               & 2      & -20850 & 2090 & 11          & 20             & 220          \\
PAM               & Genes         & Level     & 25               & 2      & -20853 & 2100 & 26          & 50             & 1300         \\
PAM               & Genes         & Level     & 50               & 2      & -20861 & 2087 & 51          & 100            & 5100         \\
PAM               & Genes         & Level     & 10               & 3      & -20905 & 2092 & 21          & 30             & 630          \\
PAM               & Genes         & Level     & 25               & 3      & -20907 & 2092 & 51          & 75             & 3825         \\
PAM               & Genes         & Level     & 50               & 3      & -20900 & 2093 & 101         & 150            & 15150        \\
PAM               & Genes         & Level     & 10               & 4      & -20898 & 2098 & 31          & 40             & 1240         \\
PAM               & Genes         & Level     & 25               & 4      & -20911 & 2094 & 76          & 100            & 7600         \\
PAM               & Genes         & Level     & 50               & 4      & -20909 & 2087 & 151         & 200            & 30200        \\
PAM               & Genes         & Level     & 10               & 5      & -20898 & 2084 & 41          & 50             & 2050         \\
PAM               & Genes         & Level     & 25               & 5      & -20909 & 2101 & 101         & 125            & 12625        \\
PAM               & Genes         & Level     & 50               & 5      & -20908 & 2093 & 201         & 250            & 50250        \\
PAM               & Path.      & Cart. & 10               & 2      & -20868 & 2090 & 20          & 11             & 220          \\
PAM               & Path.      & Cart. & 25               & 2      & -20875 & 2092 & 50          & 26             & 1300         \\
PAM               & Path.      & Cart. & 50               & 2      & -20879 & 2089 & 100         & 51             & 5100         \\
PAM               & Path.      & Cart. & 10               & 3      & -20873 & 2093 & 30          & 21             & 630          \\
PAM               & Path.      & Cart. & 25               & 3      & -20885 & 2092 & 75          & 51             & 3825         \\
PAM               & Path.      & Cart. & 50               & 3      & -20894 & 2093 & 150         & 101            & 15150        \\
PAM               & Path.      & Cart. & 10               & 4      & -20870 & 2091 & 40          & 31             & 1240         \\
PAM               & Path.      & Cart. & 25               & 4      & -20885 & 2092 & 100         & 76             & 7600         \\
PAM               & Path.      & Cart. & 50               & 4      & -20900 & 2095 & 200         & 151            & 30200        \\
PAM               & Path.      & Cart. & 10               & 5      & -20869 & 2088 & 50          & 41             & 2050         \\
PAM               & Path.      & Cart. & 25               & 5      & -20898 & 2090 & 125         & 101            & 12625        \\
PAM               & Path.      & Cart. & 50               & 5      & -20901 & 2092 & 250         & 201            & 50250        \\
PAM               & Path.      & Level     & 10               & 2      & -20904 & 2092 & 20          & 11             & 220          \\
PAM               & Path.      & Level     & 25               & 2      & -20850 & 2089 & 50          & 26             & 1300         \\
PAM               & Path.      & Level     & 50               & 2      & -20863 & 2093 & 100         & 51             & 5100         \\
PAM               & Path.      & Level     & 10               & 3      & -20886 & 2090 & 30          & 21             & 630          \\
PAM               & Path.      & Level     & 25               & 3      & -20897 & 2094 & 75          & 51             & 3825         \\
PAM               & Path.      & Level     & 50               & 3      & -20868 & 2094 & 150         & 101            & 15150        \\
PAM               & Path.      & Level     & 10               & 4      & -20873 & 2098 & 40          & 31             & 1240         \\
PAM               & Path.      & Level     & 25               & 4      & -20893 & 2090 & 100         & 76             & 7600         \\
PAM               & Path.      & Level     & 50               & 4      & -20893 & 2089 & 200         & 151            & 30200        \\
PAM               & Path.      & Level     & 10               & 5      & -20863 & 2096 & 50          & 41             & 2050         \\
PAM               & Path.      & Level     & 25               & 5      & -20910 & 2092 & 125         & 101            & 12625        \\
PAM               & Path.      & Level     & 50               & 5      & -20907 & 2089 & 250         & 201            & 50250       
\end{tabular}
\end{table*}

\end{document}